\documentclass[twoside,11pt]{article}
\usepackage{cmap}
\usepackage[T1]{fontenc}
\usepackage{lmodern}
\usepackage[utf8]{inputenc}
\usepackage{geometry} % to change the page dimensions
\usepackage[toc,page]{appendix} 
\usepackage{mathrsfs}
\usepackage{natbib}
\usepackage[font=small,labelfont=bf,labelsep=period]{caption}

\newcommand{\draftversion}{arxiv}
\usepackage{imba_preamble}

\newcommand*\samethanks[1][\value{footnote}]{\footnotemark[#1]}
\title{Class-Weighted Classification: Trade-offs and Robust Approaches}

\author{Ziyu Xu\thanks{Machine Learning Department, Carnegie Mellon University, Pittsburgh, PA 15213.}
	\and 
	Chen Dan\samethanks
	\and 
	Justin Khim\samethanks
	\and
	Pradeep Ravikumar\samethanks
}
\begin{document}
	\maketitle
	% flatex input: [tex/main.tex]
\begin{abstract}
	We address imbalanced classification, the problem in which a label may have low marginal probability relative to other labels, by weighting losses according to the correct class. 
	First, we examine the convergence rates of the expected excess weighted risk of plug-in classifiers where the weighting for the plug-in classifier and the risk may be different.
	This leads to irreducible errors that do not converge to the weighted Bayes risk, which motivates our consideration of robust risks.
	We define a robust risk that minimizes risk over a set of weightings and show excess risk bounds for this problem. 
	Finally, we show that particular choices of the weighting set leads to a special instance of conditional value at risk (CVaR) from stochastic programming, which we call label conditional value at risk (LCVaR).
	Additionally, we generalize this weighting to derive a new robust risk problem that we call label heterogeneous conditional value at risk (LHCVaR).
	Finally, we empirically demonstrate the efficacy of LCVaR and LHCVaR on improving class conditional risks.
\end{abstract}

\section{Introduction}
\label{sec:Intro}

Classification is a fundamental problem in statistics and machine learning, including scientific problems such as cancer diagnosis and satellite image processing as well as engineering applications such as credit card fraud detection, handwritten digit recognition, and text processing \citep{khan2001classification, lee2004cloud}, but modern applications have brought new challenges.
In online retailing, websites such as Amazon have hundreds of thousands or millions of products to taxonomize \citep{lin2018overview}.
 In text data, the distribution of words in documents has been observed to follow a power law in that there are many labels with few instances \citep{zipf1936psycho, feldman2019does}.
 Similarly, image data also a long tail of many classes with few examples \citep{salakhutdinov2011learning, zhu2014capturing}.
 In such settings, the classes with smaller probabilities are generally classified incorrectly more often, and this is undesirable when the smaller classes are important, such as rare forms of cancer, fraudulent credit card transactions, and expensive online purchases.
 Thus, we need modern classification methods that work well when there are a large number of classes and when the class-wise probabilities are imbalanced.

When faced with such class imbalance a popular approach in practice is to choose a metric other than zero-one accuracy, such as precision, recall, \(F_{\beta}\)-measure~ \citep{van1974foundation, van1979information}, which explicitly take class conditional risks into account, and train classifiers to optimize this metric. A difficulty with this approach however is that the right metric for imbalanced classification is often not clear. A related class of approaches keep the zero-one accuracy metric but modifies the samples instead. The popular algorithm SMOTE \citep{chawla2002smote} performs a type of data augmentation for a minority class, i.e., a class with lower probability, and sub-samples the large classes. This has led to variants with different forms of data augmentation \citep{zhou2006training, mariani2018bagan}, but from a theoretical perspective, these methods remain poorly understood.

A much simpler approach, which is also related to the approaches above, is class-weighting, in which different costs are incurred for mis-classifying samples of different labels.
Practically, this is a natural approach because it is often possible to assign different costs to different classes.
For example, the average fraudulent credit card transaction may cost hundreds of dollars, or in online retailing, failing to show a customer the correct item causes the company to lose out on the profit of selling that item.
Thus, a good classifier should be fairly sensitive to possibly fraudulent transactions, and online retailers should prioritize displaying high-profit products.
As a result, class-weighting has been studied in a variety of settings, including modifying black-box classifiers, SVMs, and neural networks \citep{domingos1999metacost, lin2002support, scott2012calibrated, zhou2006training}.
Additionally, class-weighting has been observed to be useful for estimating class probabilities, since class-weighting amounts to adjusting decision thresholds \citep{wang2008probability, wu2010robust, wang2019multiclass}.

A crucial caveat with cost-weighting however is the right choice of costs is often not clear, and with any one choice of costs, the performance of the corresponding classifier might suffer for some other, perhaps more suitable, choices of costs.

In this paper, we use cost-weighting for imbalanced classification in three ways. We start by examining a weighted sum of  class-conditional risks, i.e., the risks conditional on the class \(Y\) taking some specific value \(i\). This allows us to upweight a minority class to achieve better performance on the minority examples. We then provide an illuminating analysis of the fundamental tradeoffs that occur with any single choice of costs.

Since we may not understand precisely which weighting \(q\) to pick, we examine a robust risk that is a supremum of the weighted risks over an uncertainty set \(Q\) of possible weights. This objective can be interpreted as a class-wise distributionally robust optimization problem where we ask for robustness over the marginal distribution of \(Y\).
This leads to a minimax problem, for which we provide generalization guarantees. We also note that a standard gradient descent-ascent algorithm may solve the optimization problem when the risk is convex in the classifier parameters.

Finally, we show that for a natural class of uncertainty sets, the robust risk reduces to what call label conditional value at risk (LCVaR). We highlight a connection to conditional value at risk (CVaR), which is a well-studied quantity in portfolio optimization and stochastic programming parametrized by an \(\alpha\) in \((0, 1)\) \citep{rockafellar2000optimization, shapiro2009lectures}. 
Further, we propose a generalization that we call label heterogeneous conditional value at risk (LHCVaR) that allows for different parameters \(\alpha_{i}\) for each class \(i\).
To the best of our knowledge, this has not been examined previously, and it could possibly be used more broadly. 
To give an example in portfolio optimization, we may wish to treat risks arising from different types of assets, e.g., large-cap stocks versus small-cap stocks or domestic debt versus international debt, differently.
Next, we show that the dual form for LHCVaR is similar to that for LCVaR as long as the heterogeneity is finite-dimensional, and this leads to an unconstrained optimization problem.
Finally, we examine the efficacy of  LCVaR, and LHCVaR on real and synthetic data.

The rest of the paper is outlined as follows.
In Section~\ref{sec:Setup}, we discuss our problem setup.
In Section~\ref{sec:PluginClassification}, we examine weighting in plug-in classification. In particular, we elucidate the fundamental trade-off in weighted classification and its methodological implications.
In Section~\ref{sec:RobustProblem}, we examine a robust version of the weighted risk problem, including generalization guarantees and connections to stochastic programming.
In Section~\ref{sec:NumericalResults}, we provide numerical results, and we conclude with a discussion in Section~\ref{sec:Discussion}.
Additional proofs and results in related settings are deferred to the appendices.

%---------------------------------------------%
%---------------------------------------------%
\subsection{Further Related Work}

We briefly review other research related to imbalanced classification, but for a far more exhaustive treatment, see a survey of the area \citep{he2009learning, fernandez2018learning}.
First, two other methods may be employed to solve imbalanced classification problems.
The first is class-based margin adjustment \citep{lin2002support, scott2012calibrated, cao2019learning}, in which the margin parameter for the margin loss function may vary by class.
Broadly, margin adjustment and weighting may both be considered loss modification procedures.
The second method is Neyman-Pearson classification, in which one attempts to minimize the error on one class given a constraint on the worst permissible error on the other class \citep{rigollet2011neyman, tong2013plug, tong2016survey}. 

An important topic related to our paper but that has not been well-connected to imbalanced classification is robust optimization.
Robust optimization is a well-studied topic
\citep{ben1999robust, ben2003robust, ben2004adjustable, ben2009}.
A variant that has gained traction more recently is distributionally robust optimization \citep{ben2013, bertsimas2014, namkoong2017variance}.
Unsurprisingly, CVaR, as a coherent risk measure, has been previously connected to distributionally robust optimization \citep{goh2010distributionally}.
Distributionally robust optimization generally and CVaR specifically have also previously been used in machine learning to deal with imbalance \citep{duchi2018mixture, duchi2018learning}, but in these works, the imbalance was considered to exist in the covariates, whether known to the algorithm or not.
These are motivated by the recent push toward fairness in machine learning, in particular so that ethnic minorities  do not suffer discrimination in high-stakes situations such as loan applications, medical diagnoses, or parole decisions, due to biases in the data.

%---------------------------------------------%
%---------------------------------------------%
\section{Preliminaries}
\label{sec:Setup}

\subsection{Classification with Imbalanced Classes}
In this section, we briefly go over the problem setup.
First, we draw samples from the space \(\zspace = \xspace \times \yspace\).
For our purposes, we are interested in \(\yspace = \{0, 1\}\) or  \(\yspace = \{1, \ldots, k\}\).
Note there are two slightly different mechanisms for the data-generating process that are considered in imbalanced classification and Neyman-Pearson classification.
In the first, we are given \(n\) i.i.d.\ samples \((X_{1}, Y_{1}), \ldots, (X_{n}, Y_{n})\) from a distribution \(P_{X,Y}\).
Here, we let \(p_{i} = \prob\left(Y = i\right)\) be the probability of class \(i\).
Additionally, we sometimes refer to the vector of class probabilities as \(p\).
This is our framework of interest, since it corresponds to standard assumptions in nonparametric statistics and learning theory.
In the alternative framework, we are given \(n_{i}\) samples \((X_{1}, i), \ldots, (X_{n_{i}}, i)\) from each marginal distribution \(P_{X|Y = i}\). 
The probability of class \(i\) in this case is then known: \(p_{i} = \hat{p}_{i} = n_{i} / n\).
For the most part, these two mechanisms yield similar results, but the analyses differ slightly.
To streamline the presentation, we only consider the first case in the main paper, although we give a result for the alternative framework in the appendix that illustrates the difference.

%---------------------------------------------%
%---------------------------------------------%
\subsection{Class Conditioned Risk}

We are interested in finding a good classifier \(f: \xspace \to \mathcal{D} \supseteq \yspace\) in some function space \(\functions\), such as linear classifiers or neural networks.
In this section, we establish our risk measures of interest.
In general, we want to minimize the expectation of some loss function \(\ell: \functions \times \zspace \to [0, 1]\), which we call risk and denote \(\risk(f) = \expect[\ell(f, Z)].\)
Analogously, we define the class-conditioned risk for class \(i\) to be
\begin{equation*}
\risk_{\ell, i}(f)
=
\expect\left[\ell(f, Z) | Y = i\right].
\label{eqn:ClassConditionedRisk}
\end{equation*}
At this point, we make some observations for plug-in classification and empirical risk minimization.
In the plug-in classification results, we consider the zero-one loss \(\lzo(f, z) = \ind\{f(x) \neq y\}\), and
for our results on empirical risk minimization, we are primarily interested in convex surrogate losses. For simplicity, when \(\ell\) is clear from context, or a statement is made for a generic \(\ell\), we will denote this as \(R_{i}\).

Now, we can work toward defining weighted risks. We defined
Observe that we can relate the risk to the class-conditioned risk by
\(
\risk(f)
=
\expect 
\left[R_{Y}(f)\right]
=
\sum_{i \in \yspace} p_{i} \risk_{i}(f).
%\label{eqn:RiskAsClasses}
\)
An important part of our paper is an examination of class-weighted risk.
\begin{definition}
Let \(q = (q_{1}, \ldots, q_{|\yspace|})\) be a vector such that \(q_{i} \geq 0\) for all \(i\) and \(\expect[q_{Y}] = \sum_{i \in \yspace} q_{i} p_{i} = 1\).
Then, the \(q\)-weighted risk is
\[
R_{q}(f) 
= 
\expect 
\left[q_{Y} R_{Y}(f)\right]
=
\sum_{i \in \yspace} q_{i} p_{i} R_{i}(f).
\]
\label{def:ClassWeightedRisk}
\end{definition}
Note that the usual risk is recovered by setting \(q = (1, \ldots, 1)\).

%---------------------------------------------%
%---------------------------------------------%
\subsection{Plug-in Classification}

In this section, we discuss weighted plug-in classification.
For plug-in, we restrict our attention to the binary classification case of \(\yspace = \{0, 1\}\), and the primary quantity of interest is usually the one-zero risk \(\risk_{01}(f)\) i.e the risk under \(\ell_{0, 1}\).
In general, the risk for the best classifier is nonzero because for a given \(x\) in \(\xspace\), there is some probability it may take the value \(0\) or \(1\).

As a result, we need a way to discuss the convergence of our estimator to the best possible estimator.
We define the regression function \(\eta\) by 
\(
\eta(x) = \prob\left(Y = 1| X = x\right).
\)
Now, the Bayes optimal classifier is the classifier that minimizes the risk, and it is defined by
\(
f^{*}(x) 
= 
\ind\left\{\eta(x) > 1/2\right\}.
\)
The minimum possible risk is called the Bayes risk and denoted by \(\risk^{*} = \risk(f^{*})\), and generally we focus on minimizing the excess risk \(\excess(f) = \risk(f) - \risk^{*}\).

Following the form of the Bayes classifier, a plug-in estimator \(\hat{f}\) attempts to estimate the regression function \(\eta\) by some $\hat{\eta}$ and then ``plugs in'' the result to a threshold function.
Thus, \(\hat{f}\) has the form
\(
\hat{f}(x) = \ind\left\{\hat{\eta}(x) > 1/2\right\},
\)
which is analogous to the form of the Bayes classifier.
For additional background on plug-in estimation, see, e.g., \cite{devroye1996probabilistic}.

At this point, we wish to define the weighted versions of Bayes classifier, Bayes risk, plug-in classifier, and excess risk.
For brevity, define the threshold \(t_{q} = q_{0} / (q_{0} + q_{1})\).
First, we consider the Bayes classifier.
\begin{lemma}
Let \(q = (q_{0}, q_{1})\) be a weighting.
The Bayes optimal classifier for $q$-weighted risk is
\(
f_{q}^{*}(x) 
= 
\ind\left\{\eta(x) > t_{q}\right\}.
\)
\label{lemma:WeightedBayesClassifier}
\end{lemma}
The proof, along with proofs of other subsequent results on plug-in classification, appears in the appendix.
%~\ref{sec:PluginClassificationDetails}.
In this case, we denote the Bayes risk by \(\risk_{q}^{*} = \risk_{q}(f_q^{*})\).
Lemma~\ref{lemma:WeightedBayesClassifier} reveals that the Bayes classifier is a plug-in rule, and analogously, we see that a plug-in estimator in the weighted case takes the form
\(
\hat{f}_{q}(x) = \ind\left\{\hat{\eta}(x) > t_q\right\}.
\)
Consequently, we define excess \(q\)-risk for an empirical classifier \(\hat{f}\).
	The excess \(q\)-risk for an empirical classifier is
	\(
	\excess_q(\hat{f}) = \risk_q(\hat{f}) - R_q^{*},
	\)
and note that we are interested in bounding the expected excess \(q\)-risk for plug-in estimators.
\begin{comment}
Finally, we need to impose regularity conditions on the regression function.
A standard regularity condition is that \(\eta\) is \(\beta\)-H\"{o}lder.
\begin{definition}
A function \(f\) is \(\beta\)-H\"{o}lder if \(f(x) - f(x') \leq C\norm{x - x'}^\beta\) for $x, x' \in \xspace$ and \(C > 0\).
\end{definition}
Thus, \(\eta\) is limited to be continuous and restricted in its rate of change, which is critical in the analysis of plug-in estimators.
\end{comment}
%---------------------------------------------%
%---------------------------------------------%
\subsection{Empirical Risk Minimization}

In this section, we define empirical quantities that we need for empirical risk minimization, particularly the weighted and robust risks.
We consider \(\yspace = \{1, \ldots, k\}\).
We define the empirical class-conditioned risk by \(\hat{\risk}_{i} = (1 / N_{i}) \sum_{j = 1}^{n} \ell(f, z_{j}) \ind\left\{y_{i} = i\right\}\) where \(N_{i} = \sum_{j = 1}^{n} \ind\left\{y_{j} = {i}\right\} \).
Let \(\hat{p}_{i} = N_{i} / n\) denote the empirical proportion of observations of class \(i\), and let \(q\) be a weight vector.
The empirical \(q\)-weighted risk is
\begin{equation*}
\hat{\risk}_{q}(f)
=
\sum_{i = 1}^{k} q_{i} \hat{p}_{i} \hat{R}_{i}(f).
\label{eqn:EmpiricalqWeightedRisk}
\end{equation*}
The empirical \(Q\)-weighted risk is defined analogously by \(\hat{\risk}_{Q} = \sup_{q \in Q} \hat{\risk}_{q}(f).\)
This problem is convex in \(f\) when the loss \(\ell\) is convex and concave in \(q\) due to linearity; so one may solve the resulting saddle-point problem with standard techniques such as gradient descent-ascent, which we give in the appendix.

Often in empirical risk minimization, generalization bounds are provided, i.e., a bound on the true risk of a classifier \(f\) in \(\functions\) in terms of its empirical risk and a variance term.
To bring our results closer to those of plug-in estimation, we also consider a form of excess risk.
To distinguish the two, define the excess \((\functions, Q)\)-weighted risk to be \(\excess_{Q}(\functions) = \risk_{Q}(\hat{f}) - \risk_{Q}(f^{*}_{Q})\) where here \(\hat{f}\) is the \(Q\)-weighted empirical risk minimizer in \(\functions\) and \(f^{*}_{Q}\) is the population \(Q\)-weighted risk minimizer in \(\functions\).
Beyond the robust formulation, the key difference between excess \(q\)-weighted risk and excess \((\functions, Q)\)-weighted risk is that in the former we compete with the true regression function, and in the latter we compete with the best classifier in \(\functions\).

One additional tool we need for empirical risk minimization is a measure of function class complexity, and a typical measure of the expressiveness of a function class is Rademacher complexity.
The empirical Rademacher complexity given a sample \((X_{1}, Y_{1}), \ldots, (X_{n}, Y_{n})\) is
\[
\hat{\rademacher}_{n}(\functions)
=
\expect_{\sigma}
\sup_{f \in \functions} \sum_{i = 1}^{n} \sigma_{i} f(X_{i}),
\]
where the expectation is taken with respect to the  \(\sigma_{i}\), which are Rademacher random variables.
The Rademacher complexity is \(\rademacher_{n}(\functions) = \expect \hat{\rademacher}_{n}(\functions)
\), where the expectation is with respect to the \(X_{i}\) random variables.
 
Finally, we make one note about the loss for our empirical risk minimization results.
For binary classification, one can obtain bounds for any bounded loss function that is Lipschitz continuous in \(f(x)\).
Since we present multiclass results, we use the multiclass margin loss, which is a bounded version of the multiclass hinge loss \citep{mohri2012}.
Here, it is assumed that for each \(i\) in \(\yspace\), the function \(f\) outputs a score \(f_{i}(x)\), and the chosen class is \(\argmax_{i \in \yspace} f_{i}(x)\).
The multiclass margin loss is defined as 
\(
\marginloss(f, z)
=
\Phi\left(
f_{y}(x) - \max_{y' \neq y} f_{y'}(x)
\right)
\)
where \(\Phi(a) = \ind\left\{a \leq 0\right\} + (1 - a)\ind\left\{0 < a \leq 1\right\}\).
For simplicity, we ignore the margin parameter, usually denoted by \(\rho\), and treat it as \(1\) in our results.
Finally, we define
the projection set 
\(
\Pi_{1}(\functions)
=
\left\{x \mapsto f_{y}(x): y \in \yspace, f \in \functions\right\}.
\)

%---------------------------------------------%
%---------------------------------------------%

\section{Tradeoffs with Class Weighted Risk}
\label{sec:PluginClassification}

In this section, we examine weighted plug-in classification, and we have two main results. 
First, we show that weighted plug-in classification enjoys essentially the same rate of convergence as unweighted plug-in classification, although there is dependence on the chosen weights.
Second, there is a fundamental trade-off in that optimizing for one set of weights \(q\) may lead to suboptimal performance for another set of weights \(q'\).

%---------------------------------------------%
%---------------------------------------------%

\subsection{Excess Risk Bounds}

We start with the excess risk bound for plug-in estimators when the weighting is well-specified.

\begin{proposition}
Suppose the regression function $\eta$ is $\beta$-H\"older.
Then, the $q$-weighted excess risk of $\hat{f}_q$ satisfies
\label{theorem:WeightedRiskConvergence}
\begin{align*}
    \expect\excess_q(\hat{f}_q) \leq O\left((q_0 + q_1) n^{-\frac{\beta}{2\beta + d }}\right).
\end{align*}
\label{prop:ExcessqRisk}
\end{proposition}

Here, we see that the upper bound depends linearly on \(q_{0}\) and \(q_{1}\).
This implies that when we increase the weight for a class with few examples, then our bound on the excess risk increases. While previous cost weighting setups have normalized the sum of weights \cite{scott2012calibrated}, our normalization scheme is computed with respect to prior probabilities on each class as well, and consequently we explicitly include \(q_0, q_1\) in our bound. Our choice of domain for weights is defined in \cref{sec:RobustProblem}.

Now, we turn to our second task: examining the weighted excess risk of the $\hat{f}_q$ under a different weighting $q'$.
Observe that we can decompose the excess risk as
\begin{align}
    \expect \excess_{q'}(\hat{f}_q) 
    &= 
    \underset{\text{estimation error}}{\underbrace{\expect R_{q'}(\hat{f}_q) - R_{q'}(f^*_q)}} + \underset{\text{irreducible error}}{\underbrace{R_{q'}(f^*_q) - R_{q'}(f^*_{q'})}} \nonumber \\
    &=: 
    \estimation + \irreducible.
    \label{eqn:AltWeightDecompositon}
\end{align}
\begin{figure}[t]
    \centering
    \includegraphics[width=\columnwidth	]{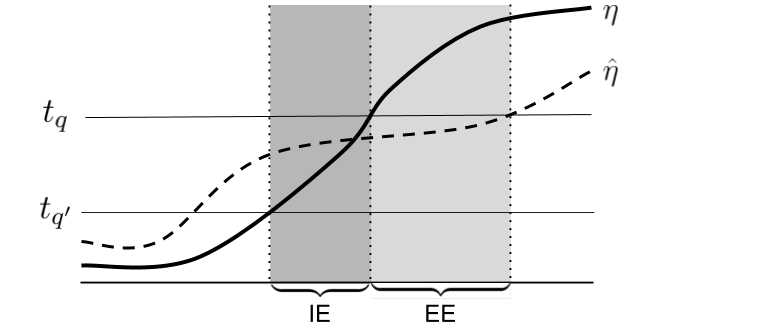}
    \caption{The irreducible error (IE) and estimation error (EE). 
    The irreducible error is the measure of the set of \(x\) where $\eta(x)$ is between thresholds of $q'$ and $q$, which does not depend on $\hat{\eta}$. 
    The estimation error is the measure of the \(x\) for which $\hat{\eta}(x)$ and $\eta(x)$ lead to different plug-in estimates.}
    \label{fig:ExcessqprimeRiskDiagram}
\end{figure}

Unsurprisingly, we see that an error term that is constant, or ''irreducible'' appears in equation~\eqref{eqn:AltWeightDecompositon}. 
Then, we see the irreducible error is given by the measure of the subset of \(\xspace\) where \(\eta(x)\) lies between \(t_{q}\) and \(t_{q'}\). Given that we know the Bayes optimal classifier for any weighting, we observe that the irreducible error can be upper bounded by a term proportional to the the product of the measure of \(P_X\) in the region between \(t_q\) and \(t_{q'}\), and the difference between the thresholds themselves.
We state this formally in the following proposition.

\begin{proposition}
Let \(\underline{t}_{q, q'} = \min\{t_{q}, t_{q'}\}\) and \(\overline{t}_{q, q'} = \max\{t_{q}, t_{q'}\}\). 
The irreducible error satisfies the bound
\begin{align*}
    \irreducible
	%\expect\left[R_{q'}(f_q^*) - R_{q'}(f_{q'}^*)\right] 
	\leq (q'_0 + q_1')\left|t_q - t_{q'}\right|\prob\left(\underline{t}_{q, q'} \leq \eta(X) \leq \overline{t}_{q, q'} \right)
	%\label{eqn:IrreducibleErrorBound}
\end{align*}
\label{prop:IrreducibleError}
\end{proposition}

A visualization is given in Figure~\ref{fig:ExcessqprimeRiskDiagram}.
Now, we turn to analyze the estimation error.
The result is in many ways similar to Proposition~\ref{prop:ExcessqRisk}, but an additional term appears due to the decision threshold \(t_{q}\) for \(\hat{\eta}\) differing from that of the risk measurement \(t_{q'}\).
\begin{proposition} 
%Let $\eta$ be $\beta$-H\"older.
For any density estimator $\hat{\eta}$, the estimation error satisfies
	\ificml{
	\begin{align*}
	\estimation
	\leq &
	(q_0' + q_1') \expect\left[\norm{\eta - \hat{\eta}}_{L_1(P_X)}\right] + \\
	&(q_0' + q_1') \left|t_{q'} - t_q\right| 
	\expect\left[\prob\left(\hat{f}_q(x) \neq f^*_q(x)\right)\right]
	\end{align*}
	}{
	\begin{align*}
	\estimation
	\leq 
	(q_0' + q_1')\left(O\left(n^{-\frac{\beta}{2\beta + d}}\right) + \left|t_{q'} - t_q\right| 
	\expect\left[\prob\left(\hat{f}_q(x) \neq f^*_q(x)\right)\right]\right)
	\end{align*}
	}
\label{prop:ExcessqprimeRisk}
\end{proposition}
\begin{corollary}\label{corollary:ExcessqprimeRiskHolder}
When $\eta$ is $\beta$-H\"older, using local polynomial estimator \cite{yang1999minimax} for $\hat{\eta}$ gives
	\ificml{
	\begin{align*}
        \estimation
        \leq &
        (q_0' + q_1') O\left(n^{-\frac{\beta}{2\beta + d}}\right) + \\
        &(q_0' + q_1') \left|t_{q'} - t_q\right| 
        \expect\left[\prob\left(\hat{f}_q(x) \neq f^*_q(x)\right)\right]
    \end{align*}
	}{
	\begin{align*}
	\estimation
	\leq &
	(q_0' + q_1') O\left(n^{-\frac{\beta}{2\beta + d}}\right) + (q_0' + q_1') \left|t_{q'} - t_q\right| 
	\expect\left[\prob\left(\hat{f}_q(x) \neq f^*_q(x)\right)\right]
	\end{align*}
	}
\end{corollary}
Consequently, we can upper bound the expected excess \(q'\)-risk.
The probability in the bound of the estimation error has been considered in the context of nearest neighbors \citep{chaudhuri2014rates}, but in general, additional assumptions are required to provide an explicit rate. We consider one such assumption in the appendix.

%---------------------------------------------%
%---------------------------------------------%
\section{Robust Class Weighted Risk}
\label{sec:RobustProblem}
Based the results in the previous section, we know that the performance degradation need not be graceful when we don't know how to choose the weights. This motivates us to study a more robust version of class weighted risk.
\begin{definition}
	Let \(Q \subseteq \reals^{|\yspace|}\) be a compact convex set such that \(q_{i} \geq 0\) for each \(i\) and  \(\expect[q_{Y}] =  1\) for each \(q\) in \(Q\).
	Then, the \(Q\)-weighted risk is
	\begin{align*}
	&\begin{aligned}
	\risk_{Q}(f)
	&=
	\sup_{q \in Q}
	\expect 
	\left[q_{Y} R_{Y}(f)\right]
	=
	\sup_{q \in Q}
	\sum_{i \in \yspace}
	q_{i}p_{i} \risk_{i}(f).
	\end{aligned}
	\end{align*} 
	Additionally, we refer to the set \(Q\) as the uncertainty set.
	\label{def:RobustRisk}
\end{definition}

In this section, we have two goals: (1) to provide excess \(\functions\)-risk bounds and generalization bounds for robust weighted risk via uniform convergence and (2) to make connections to stochastic optimization via special choices of uncertainty set.
We start with generalization; the proofs are given in the appendix. %Section~\ref{sec:RobustWeightingProofs}.

\begin{theorem}
Let \(\ell = \marginloss\) be the multiclass margin loss.
Recall that \(N_{i} = \sum_{j = 1}^{n} \ind\left\{y_{j} = i\right\}\).
With probability at least \(1 - \delta\), we have the generalization bound
\ificml{
\begin{align*}
& \begin{aligned}
\risk_{Q}(f)
&\leq 
\sup_{q \in Q}
\left\{
\hat{\risk}_{q}(f)
+
\sum_{i = 1}^{k} q_{i} p_{i}
\right. \\ & \qquad \left. \times
\left(
4k \expect\left[\frac{N_{i}}{p_{i}n} \hat{\rademacher}_{N_{i}}(\Pi_{1}(\functions))
\right]
+
\sqrt{\frac{\log \frac{k}{\delta}}{2 p_{i}^{2}n}}
\right)
\right\}
\end{aligned}
\end{align*}}{
\begin{align*}
& \begin{aligned}
\risk_{Q}(f)
&\leq 
\sup_{q \in Q}
\left\{
\hat{\risk}_{q}(f)
+
\sum_{i = 1}^{k} q_{i} p_{i} \times
\left(
4k \expect\left[\frac{N_{i}}{p_{i}n} \hat{\rademacher}_{N_{i}}(\Pi_{1}(\functions))
\right]
+
\sqrt{\frac{\log \frac{k}{\delta}}{2 p_{i}^{2}n}}
\right)
\right\}
\end{aligned}
\end{align*}
}
for every \(f\) in \(\functions\) and the excess risk bound
\ificml{
\begin{align*}
\excess_{Q}(\functions)
&\leq 
 2
\sup_{q \in Q}
\sum_{i = 1}^{k} q_{i} p_{i}
\\ & \qquad \times
\left(
8k \expect\left[\frac{N_{i}}{p_{i}n} \hat{\rademacher}_{N_{i}}(\Pi_{1}(\functions))
\right]
+
\sqrt{\frac{\log \frac{k}{\delta}}{2p_{i}^{2} n}}
\right). 
\end{align*}
}{
\begin{align*}
\excess_{Q}(\functions)
&\leq 
2
\sup_{q \in Q}
\sum_{i = 1}^{k} q_{i} p_{i}
 \times
\left(
8k \expect\left[\frac{N_{i}}{p_{i}n} \hat{\rademacher}_{N_{i}}(\Pi_{1}(\functions))
\right]
+
\sqrt{\frac{\log \frac{k}{\delta}}{2p_{i}^{2} n}}
\right). 
\end{align*}
}

\label{theorem:MainStandardSampling}
\end{theorem}
A few remarks are in order.
First, note that we only use the multiclass margin loss because it leads to simple multiclass bounds.
In a binary classification setting, standard results would imply generalization for other Lipschitz losses.
Second, in many cases, we can simplify the Rademacher complexity term.
The following result applies to commonly-used function classes such as linear functions and neural networks \citep{bartlett2017, golowich2018, mohri2012}. 

\begin{corollary}
Let \(\ell = \marginloss\) be the multiclass margin loss.
Let \(\functions\) be a function class satisfying
\(
\hat{\rademacher}_{n}(\Pi_{1}(\functions) )
\leq 
C(\functions) n^{-1/2}
\)
for some constant \(C(\functions)\) that does not depend on \(n\).
Then with probability at least \(1 - \delta\), we have the generalization bound
\ificml{
\begin{align*}
&\begin{aligned}
\risk_{Q}(f)
&\leq \sup_{q \in Q}
\left\{
\hat{\risk}_{q}(f)
+
\sum_{i = 1}^{k} 
q_{i} p_{i}
\right. \\ & \qquad \times \left.
\left(
\frac{4k C(\functions)}{\sqrt{p_{i} n}}
+
\sqrt{\frac{\log \frac{k}{\delta}}{2p_{i}^{2} n}}
\right)
\right\} 
\end{aligned}
\end{align*}
}{
\begin{align*}
&\begin{aligned}
\risk_{Q}(f)
&\leq \sup_{q \in Q}
\left\{
\hat{\risk}_{q}(f)
+
\sum_{i = 1}^{k} 
q_{i} p_{i} \times \left(
\frac{4k C(\functions)}{\sqrt{p_{i} n}}
+
\sqrt{\frac{\log \frac{k}{\delta}}{2p_{i}^{2} n}}
\right)
\right\} 
\end{aligned}
\end{align*}
}
and the excess \((\functions, q)\)-risk bound
\[
\excess_{Q}(\functions)
\leq 
2
\sup_{q \in Q}
\sum_{i = 1}^{k} q_{i} p_{i}
\left(
\frac{8k C(\functions)}{\sqrt{p_{i} n}}
+
\sqrt{\frac{\log \frac{k}{\delta}}{2p_{i}^{2} n}}
\right).
\]
\label{cor:MainStandardSampling}
\end{corollary}

%---------------------------------------------%
%---------------------------------------------%
\subsection{Connections to Stochastic Programming}

In this section, we make concrete connections to stochastic programming \citep{shapiro2009lectures}.
First, we introduce label conditional value at risk, and then we describe the generalization, label heterogeneous conditional value at risk.

\subsubsection{Label CVaR}
\label{subsubsec:ConnectionCVaR}

We start with the definition.

\begin{definition}
Let \(\alpha\) in \((0, 1)\) be given.
Define the set 
\(
Q_{\alpha} 
= 
\left\{
q :  \expect[q_Y] = 1,  q_i \in \left[0, \alpha^{-1}\right] \text{ for }i \in 1, \dots, k
\right\}.
\)
The label conditional value at risk (LCVaR) is
\(
\lcvar_{\alpha}(f)
=
\risk_{Q_{\alpha}}(f).
\)
\label{def:CVaR}
\end{definition}

Now, we describe the connection to CVaR.
Letting \(Z\) be a random variable, the CVaR of \(Z\) at level \(\alpha\) is 
\(\cvar_{\alpha}(Z) = \sup_{Q \in Q_{\alpha}^{*}} \expect_{Q}[Z] =  \sup_{Q \in Q_{\alpha}^{*}} \expect[(dQ/dP) Z], \) where \(Q_{\alpha}^{*}\) is the set of all probability measures that are absolutely continuous with respect to the underlying measure \(P\) such that \(dQ/dP \leq \alpha^{-1}\).
If \(Z\) takes values on a finite discrete probability space with probability mass function \(p\), then the CVaR may be written as
\(\cvar_{\alpha}(Z) = \sup_{q \in Q_\alpha} \sum_{i = 1}^{k} q_{i}p_{i} Z.\)
Thus, LCVaR is a specialization of CVaR to the variables \(R_{Y}(f)\), which take values on the finite discrete space \(\yspace\).
Notably, this is in contrast to other uses of CVaR in machine learning where, as noted previously, CVaR is used with respect to samples directly, in order to provide robustness or fairness.
As with CVaR, LCVaR is a straightforward way to provide robustness.
Intuitively, it moves weight to the worst losses, where all weightings are bounded by the same constant \(\alpha^{-1}\).
Now, we consider the dual form.

\begin{proposition}[LCVaR dual form]
LCVaR permits the dual formulation
\begin{align*}
&\begin{aligned}
\lcvar_{\alpha}(f)
&=
\inf_{\lambda \in \reals} 
\left\{
\frac{1}{\alpha}
\expect[(R_{Y}(f) - \lambda)_{+}]
+
\lambda
\right\}.
\end{aligned}
\end{align*}
Moreover, if \(\functions\) is compact in the supremum norm on \(\xspace\) and \(\ell\) is continuous, then the dual form holds for all \(f\) in \(\functions\).
\label{prop:CVaRDual}
\end{proposition}

The proof is mostly standard and therefore deferred to the appendix.
The only trick compared with CVaR is showing that we may restrict the domain of \(\lambda\) to a compact set; which essentially requires showing that the process \(\{R_{Y}(f): f \in \functions\}\) is sufficiently well-behaved.
It would also suffice to assume that \(\ell\) is bounded, as with most theoretical results in learing theory.
Note that to minimize LCVaR, we can solve this convex program in \(\lambda\) and \(f\).

%---------------------------------------------%
%---------------------------------------------%
\subsubsection{Label Heterogeneous CVaR}

While the LCVaR approach of the previous section is useful for providing some robustness in a computationally tractable manner, it may not be best suited for imbalanced classification because it treats all classes identically in that each \(q_{i}\) must lie in the interval \([0, \alpha^{-1}]\).
Since imbalanced classification is inherently a problem of heterogeneity, we may wish to allow \(q_{i}\) to be in some interval \([0, \alpha^{-1}_{i}]\) instead.
We can formalize this problem as follows.
\begin{definition}
Define the uncertainty set 
\(
Q_{H, \alpha}
=
\left\{
q: 
\expect [q_{Y}] = 1,
q_{i} \in [0, \alpha^{-1}_{i}]
\text{ for } i = 1, \ldots, k
\right\}.
\)
We call the resulting optimization problem label heterogeneous conditional value at risk (LHCVaR), and we write
\begin{align*}
& \begin{aligned}
\lhcvar_{\alpha}(f)
=
\sup_{q \in Q_{H, \alpha}}
\expect\left[q_{Y} R_{Y}(f)\right].
\end{aligned}
\end{align*}
\label{def:LHCVaR}
\end{definition}

Similar to LCVaR, this has a dual form.

\begin{proposition}
A dual form for LHCVaR is given by
\begin{align*}
& \begin{aligned}
\lhcvar_{\alpha}(f)
=
\inf_{\lambda \in \reals}
\expect\left[
\alpha_{Y}^{-1}
\left(R_{Y}(f) - \lambda\right)_{+}
\right] 
+
\lambda.
\end{aligned}
\end{align*}
Moreover, if \(\functions\) is compact in the supremum norm on \(\xspace\) and \(\ell\) is continuous, then the dual form holds for all \(f\) in \(\functions\).
\label{prop:HCVaRDual}
\end{proposition}

Again, we note that an alternative sufficient condition for the dual to hold for all \(f\) in \(\functions\) is that \(\ell\) be bounded.
Importantly, the label heterogeneous CVaR dual form is convex in \(f\) and \(\lambda\).
As a result, we can still optimize efficiently, in principle.

We also note that the finite dimension \(k\) is crucial for label heterogeneous CVaR.
This is due to our use of the minimax theorem, which requires compactness in various places; so in general this result cannot be extended to the infinite-dimensional case.

\begin{figure*}[th]
	\centering
	\begin{subfigure}{0.33\textwidth}
		\includegraphics[width=\columnwidth]{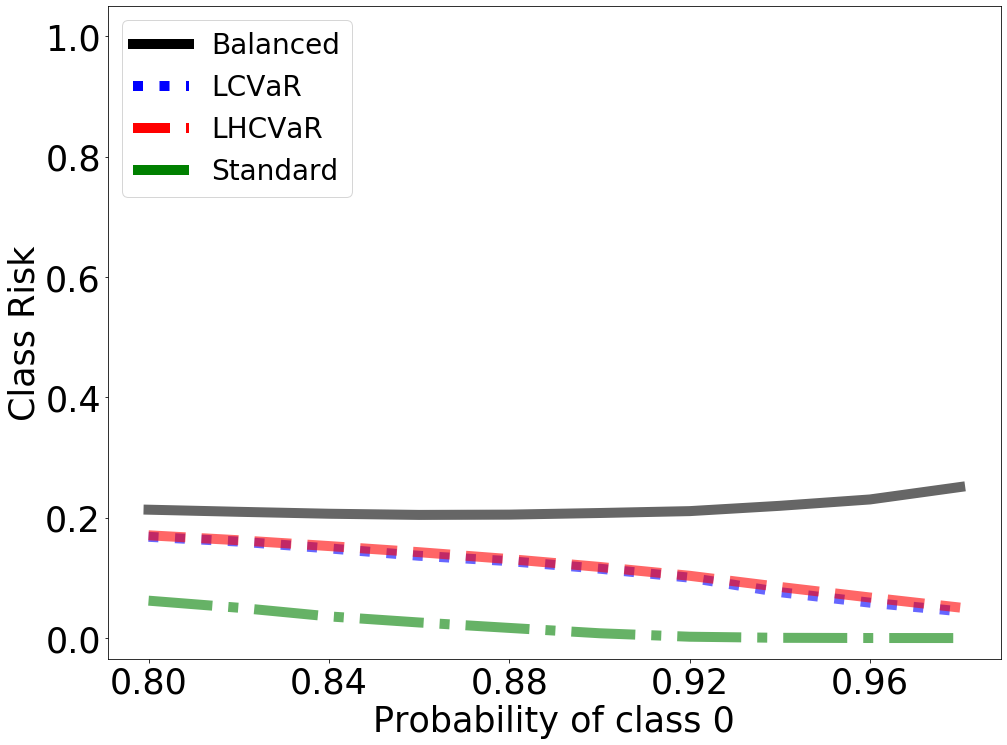}
		\vspace{-15pt}
		\caption{Class 0 risk}
		\label{subfig:SyntheticClasswiseZero}
	\end{subfigure}\begin{subfigure}{0.33\textwidth}
		\includegraphics[width=\columnwidth]{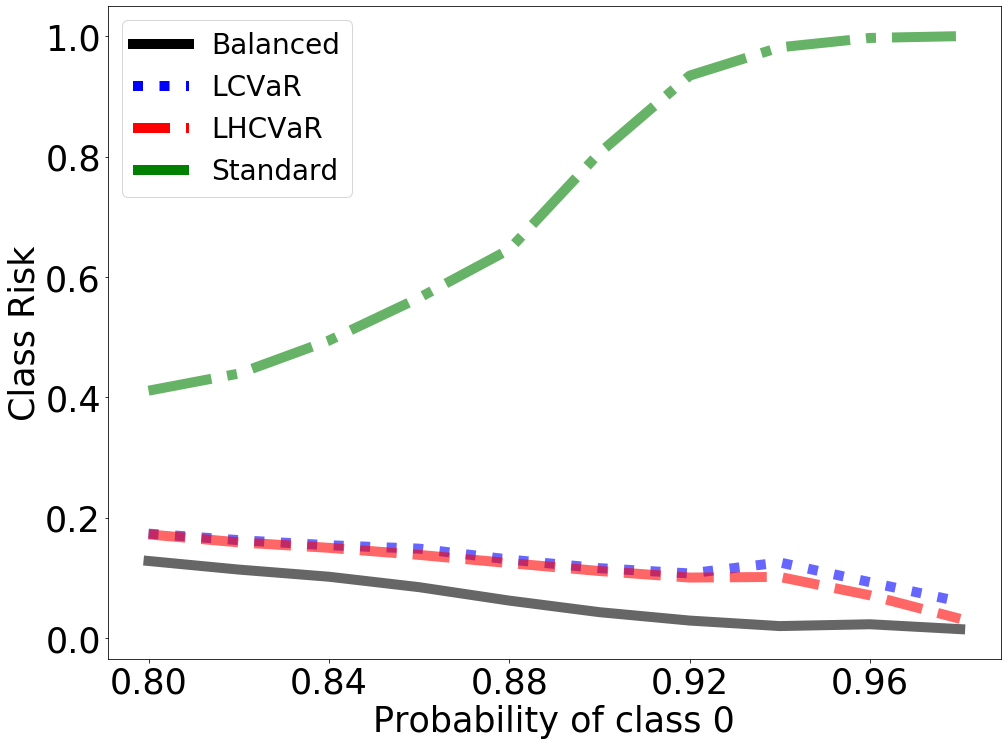}
		\vspace{-15pt}
		\caption{Class 1 risk}
		\label{subfig:SyntheticClasswiseOne}
	\end{subfigure}
	\begin{subfigure}{0.33\textwidth}
		\includegraphics[width=\columnwidth]{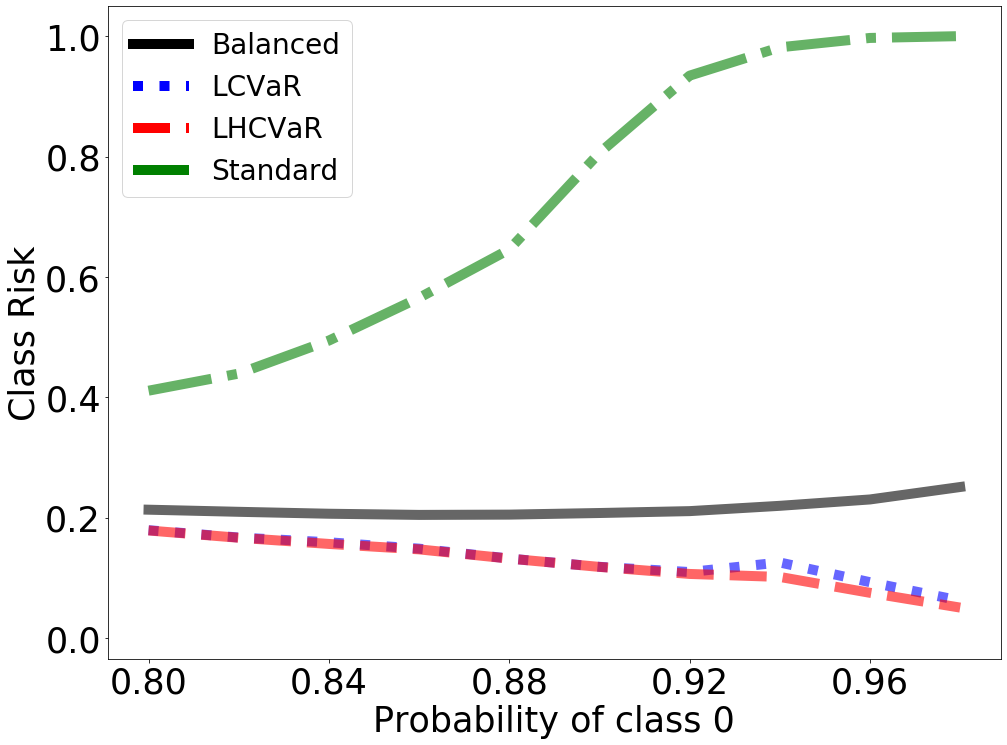}
		\vspace{-15pt}
		\caption{Worst class risk}
		\label{subfig:SyntheticClasswiseWorst}
	\end{subfigure}
	\caption{Plots of class 0, class 1, and worst class risk on the test dataset under different choices of \(1 - p\) in the synthetic experiment. The worst test class risk is the maximum of the risks of the two classes for each choice of the probability of class 0. LCVaR and LHCVaR performs better in worst class risk than both standard and balanced risks as class imbalance increases.}
	\label{fig:SyntheticResults}
\end{figure*}

\section{Numerical Results}
\label{sec:NumericalResults}

\subsection{Methods}
\label{subsec:Methodologies}

We examine the empirical performance of LCVaR and LHCVaR risks, and compare them against the standard risk and a balanced risk as baselines. Let \(\hat{p}_i\) be the empirical proportion of the \(i\)th label and \(\hat{R}_i\) be the empirical class conditional risk.

\paragraph{Balanced risk} Here, we consider the specific weighting where each class is equally weighted:
\begin{align*}
	\hat{R}_{1/(k \hat{p})}(f) = \frac{1}{k}  \sum_{i=1}^{k} \emprisk_{i}(f)
\end{align*}
i.e., we fix \(q_{i} = 1 / (k \hat{p}_i)\).

\paragraph{LCVaR} The empirical formulation optimizes the dual formulation, in which \(\alpha\) is a hyperparameter:
\begin{align}
	\hat{\lcvar}_\alpha(f) = 
	\underset{\lambda \in \reals}{\min\ }\left\{\frac{1}{\alpha}\sum\limits_{i = 1}^k\hat{p}_i(\hat{R}_{i}(f) - \lambda)_{+}
	+
	\lambda
	\right\}.
	\label{eqn:EmpiricalLCVaR}
\end{align}

\paragraph{LHCVaR} We similarly optimize a dual form in the empirical LHCVaR risk. To reduce the number of hyperparameters to only \(c \in (0, 1]\) and \(\temp \in (0, \infty)\), we calculate \(\alpha_i\) as follows: 
\begin{align}
	\alpha_i^{(\temp, c)} = c \left( \frac{\hat{p_i}^{1 / \temp}}{\sum_{j = 1}^{k} \hat{p_j}^{1 / \temp}}\right).
	\label{eqn:EmpiricalLHCVaR}
\end{align}

\(\temp\) behaves as a temperature parameter (similar to \citealt{jang2016categorical, wang2020balancing}) and causes \(\alpha\) to become a smoother distribution of weights when \(\temp > 1\) and converge to uniform weights as \(\temp \rightarrow \infty\). Conversely, when \(\temp < 1\), the alpha distribution becomes sharper and heavily weights the classes with lowest \(\hat{p}_i\) as \(\temp \rightarrow 0\). We simply choose a \(\kappa\) of 1 unless otherwise stated.  \(c\) consequently characterizes the total magnitude of the weights. Ultimately, we formulate the empirical risk as:
\begin{align*}
	\hat{\lhcvar}_{\temp, c}(f) 
	= 
	\inf_{\lambda \in \reals}\left\{\sum\limits_{i = 1}^k\frac{\hat{p}_i}{\alpha_i^{(\temp, c)}}(\hat{R}_i(f) - \lambda)_+ + \lambda \right\}
\end{align*}

We train a logistic regression model with gradient descent on a cross entropy loss, which acts as a convex surrogate loss for zero-one risk.

\subsection{Datasets}
\label{subsec:Experiments}

We evaluate our methods on both synthetic and real datasets. 

\paragraph{Synthetic Datasets} The data in our synthetic experiment is constructed for \(\xspace = [0, 1]\) and  \(\yspace=\{0, 1\}\). For a given \(p = P(Y = 0)\), we generated a dataset by uniformly randomly sampling an \(X\) in \([0, 1]\) and sampling a \(Y\) with the following distribution:
\begin{align*}
	P(Y = 1\mid X = x) &= x^{\frac{p}{1 - p}}\\
	P(Y = 0\mid X = x) &= 1 - x^{\frac{p}{1 - p}}.
\end{align*} 

In these synthetic datasets, we note that the Bayes optimal classifier and class risks are:
\begin{align*}
f^*(x) &= \ind\left\{x > \left(\frac{1}{2}\right)^{ \frac{1 - p}{p}}\right\}\\
R_{0}(f^*) &=1 - (1 + p)\left(\frac{1}{2}\right)^{\frac{1}{p}}\\
R_{1}(f^*) &= \left(\frac{1}{2}\right)^{\frac{1}{p}}.
\end{align*} 
When \(p\) is high, \(R_0(f^*) < R_1(f^*)\), which leads to a classifier that has vastly worse performance on class 1 compared to class 0. This discrepancy in class risk is a common issue in classification problems where there is a significant class imbalance. 

We randomly generated 100,000 data points for both train and test sets. We generated datasets for each value of \(p\) from 0.80 to 0.98, inclusive, in steps of 0.02.

\paragraph{Real World Datasets} We also experiment on the Covertype dataset taken from the UCI dataset repository \cite{dua2017uci}.  This dataset is 53-dimensional with 7 classes and has 2\%-98\% (11340-565892 examples) train-test split.

\begin{figure*}
	\centering
	\begin{subfigure}[t]{0.33\textwidth}
		\includegraphics[width=\textwidth]{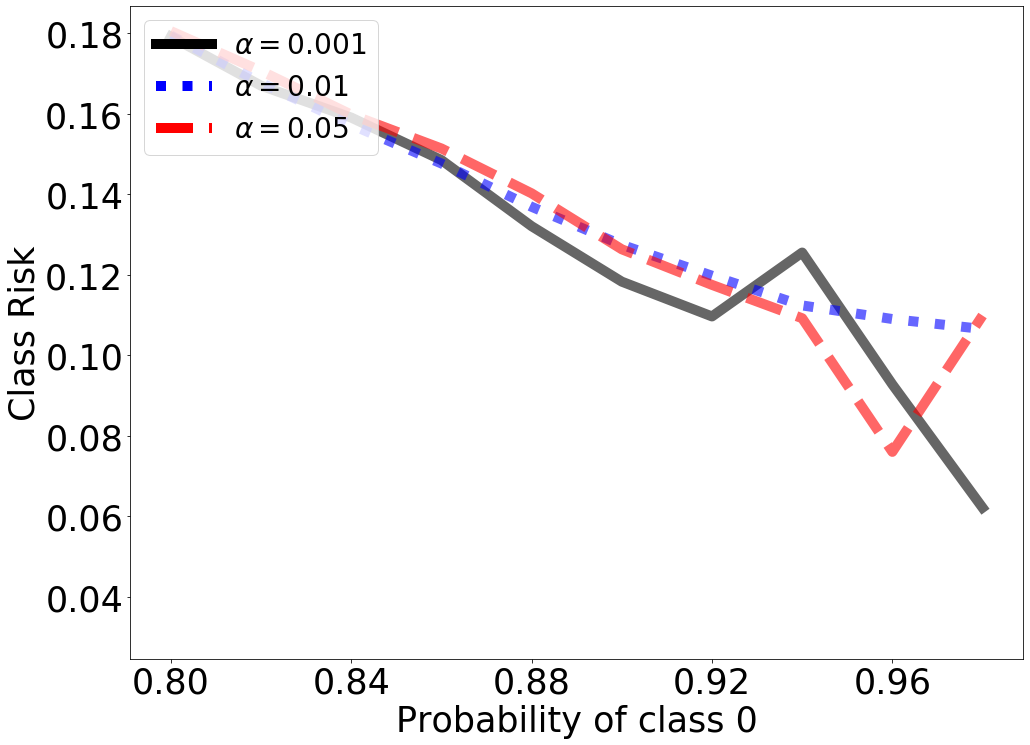}
		\caption{Varying \(\alpha\) for LCVaR}
	\end{subfigure}
	\begin{subfigure}[t]{0.33\textwidth}
		\includegraphics[width=\textwidth]{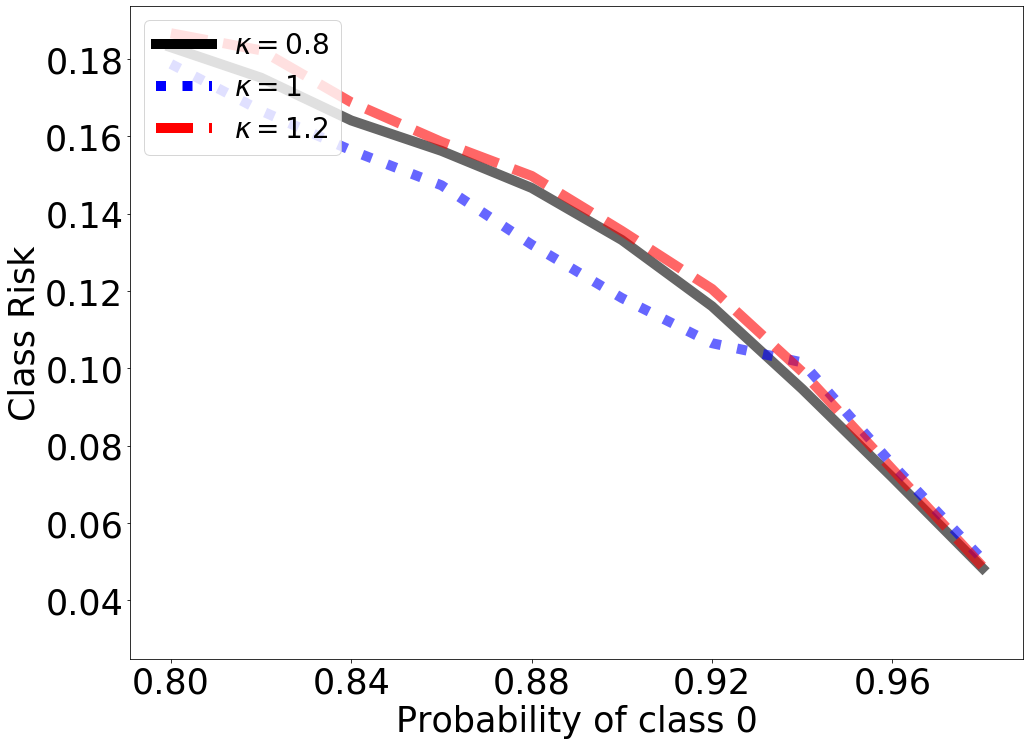}
		\caption{Varying \(\kappa\) for LHCVaR with \(\alpha=0.01\)}
	\end{subfigure}
	\caption{Worst class risk of different \(\alpha\) values for LCVaR and \(\kappa\) values for LHCVaR in the synthetic setting. Across different levels of class imbalance, \(\alpha\) and \(\kappa\) do not have a significant impact on worst class risk of LCVaR and LHCVaR.}
	\label{fig:SyntheticAblation}
\end{figure*}

\begin{figure*}[bth]
	\centering
	\begin{subfigure}[t]{0.25\textwidth}
		\includegraphics[width=\textwidth]{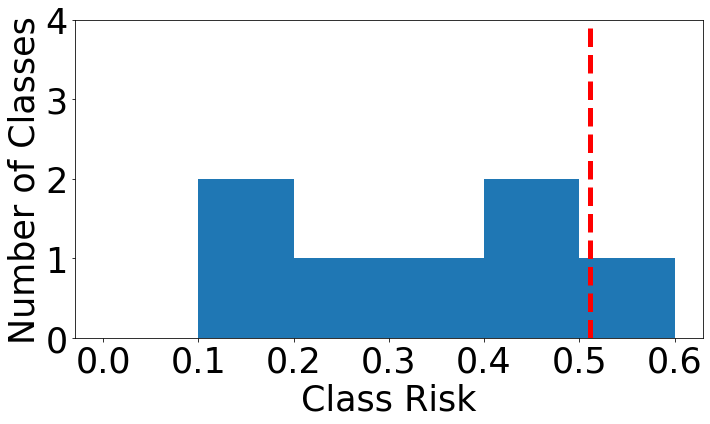}
		\vspace{-15pt}
		\caption{Standard}
	\end{subfigure}\begin{subfigure}[t]{0.25\textwidth}
		\includegraphics[width=\textwidth]{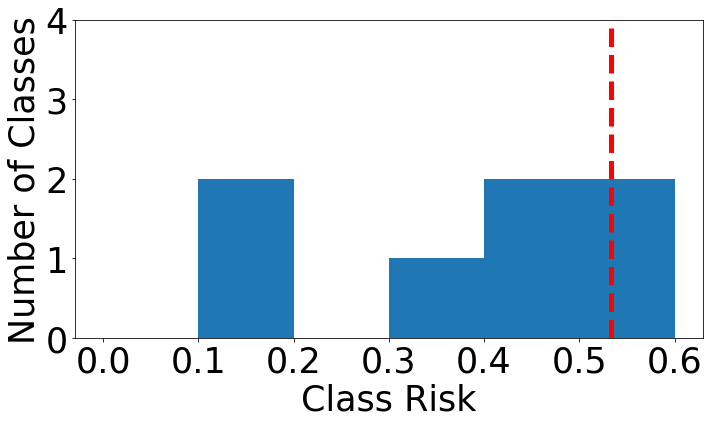}
		\vspace{-15pt}
		\caption{Balanced}
	\end{subfigure}\begin{subfigure}[t]{0.25\textwidth}
		\includegraphics[width=\textwidth]{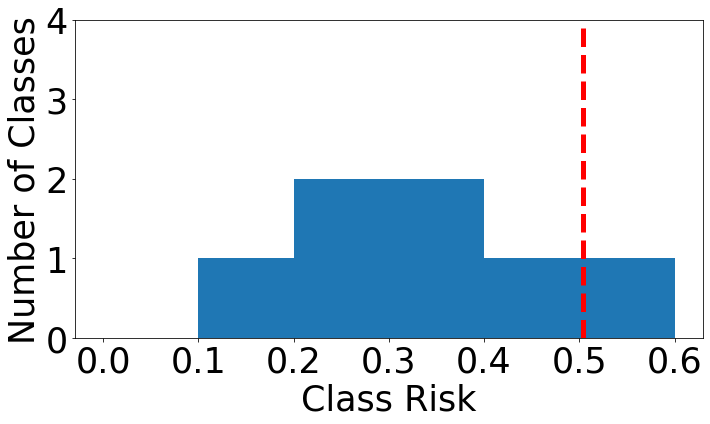}
		\vspace{-15pt}
		\caption{LCVaR}
	\end{subfigure}\begin{subfigure}[t]{0.25\textwidth}
		\includegraphics[width=\textwidth]{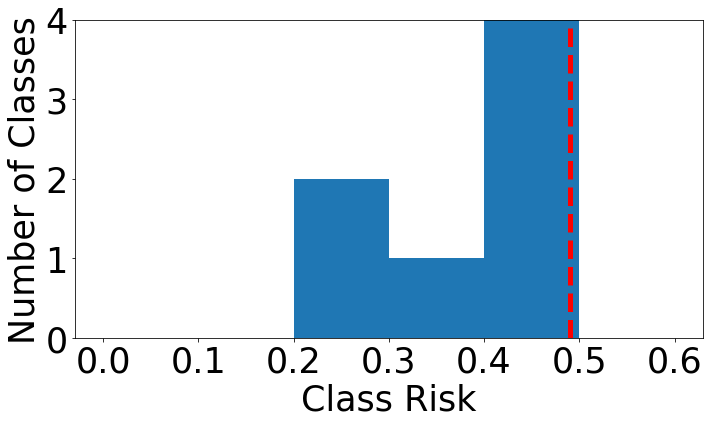}
		\vspace{-15pt}
		\caption{LHCVaR}
	\end{subfigure}
	\caption{Histogram of class risks for each method on the Covertype dataset. The red line marks the largest risk for each method. The distribution of class risks for standard and balanced methods are more spread out, while the class risks for LCVaR and LHCVaR are more concentrated near the max class risk. The max class risks are slightly lower for LCVaR and LHCVaR compared to the other two methods.}
	\label{fig:CovtypeHists}
\end{figure*}

\subsection{Results}

\paragraph{Synthetic} 
In \cref{fig:SyntheticResults}, we can 
observe that the the worst case class risk of LCVaR and LHCVaR across multiple values of \(p\) is better than both the standard and balanced classifier. The classwise risks of LCVaR and LHCVaR are relatively close across different values of \(p\), while there is a large discrepancy between classwise risks of the classifier trained under the standard or balanced risks. Note that the more significant the imbalance, i.e., the smaller the \(p\), the better LCVaR and LHCVaR perform compared to balanced risk on class 0, while paying a progressively smaller price on the class 1 risk. The same is also true between both LCVaR and LHCVaR and the standard risk, although with the classes swapped. We note that while the worst class risk of LCVaR and LHCVaR seem to decrease with greater imbalance, this may not be a general property of these methods. Rather, this is more likely an artifact of the synthetic setup having more probability mass further from the decision boundary as the imbalance increases. The main observation is simply that LCVaR and LHCVaR have lower worst class risk in comparison to the baseline methods. Thus, this empirically demonstrates that both LCVaR and LHCVaR can significantly improve the highest class risks while losing little in performance on classes with lower risks.

In addition to comparing against baselines, we also examine the effect of different choices of \(\alpha\) and \(\kappa\) on LCVaR and LHCVaR, respectively. The results of this comparison are in \cref{fig:SyntheticAblation}. In both methods, varying the hyperparameters does not have a dramatic impact on the behavior of the worst class risk for both these methods across different values of class imbalance.

\begin{table}[ht]
	\centering
	\caption{Standard risk and risk of the worst class for each method on the Covertype dataset. LCVaR and LHCVaR improve on the worst class risk.}
	\vspace{-5pt}
	\begin{tabular}{c|c|c}
		\textbf{Method} & \textit{Standard Risk} & \textit{Worst Class Risk}\\
		\hline
		LHCVaR & 0.3979 & \textbf{0.4907}\\
		LCVaR & 0.3384 & 0.5037\\
		Standard & \textbf{0.3275} & 0.5111\\
		Balanced & 0.3765 & 0.5333 
	\end{tabular}
	\label{table:CovtypeRiskTable}
\end{table}

\begin{table}[h]
	\centering

	\caption{Performance of LCVaR across different $\alpha$ values, and LHCVaR across different \(\kappa\) values. The performance each method is relatively agnostic to  choices of \(\alpha\) and \(\kappa\), although the smallest  choices of \(\alpha\) and \(\kappa\) for each method have the largest changes in worst class risk, respectively. }
	\begin{tabular}{c|c|c|c|c}
		\textbf{Method}         & \textit{$\alpha$} & \textit{$\kappa$} & \textit{Standard Risk} & \textit{Worst Class Risk} \\ \hline
		\multirow{3}{*}{LCVaR}  & 0.01              & N/A               & 0.4266                 & 0.5474                    \\
		& 0.05              & N/A               & 0.3993                 & 0.4932                    \\
		& 0.1               & N/A               & 0.4060                 & 0.5037                    \\ \hline
		\multirow{3}{*}{LHCVaR} & 0.05              & 0.8               & 0.4308                 & 0.5408                    \\
		& 0.05              & 1                 & 0.3979                 & 0.4907                    \\
		& 0.05              & 1.2               & 0.4171                 & 0.5050                   
	\end{tabular}
	\label{table:CovtypeAblation}
\end{table}

\paragraph{Real} 
In \cref{table:CovtypeRiskTable}, we observe that LCVaR and LHCVaR have better worst class risks than the standard and class weighted baselines. However, improving worst class risk comes at a cost to to the standard risk in the case of both LCVaR and LHCVaR. This tradeoff is reflected in the histograms of class risk shown in \cref{fig:CovtypeHists}, where the class risks under the standard and balanced classifiers are more spread out and have classes with much lower risks. On the other hand, LCVaR and LHCVaR have class risk distributions that are more concentrated towards the worst class risk value. Consequently, LCVaR and LHCVaR achieve a lower worst class risk, which is consistent with our theory.

We also compare the effect of choosing different \(\alpha\) and \(\kappa\) on LCVaR and LHCVaR, respectively, in \cref{table:CovtypeAblation}. We see that the worst class risk still performs well under different choices of \(\alpha\) and \(\kappa\), although there is some degradation when the \(\alpha\) is smaller than optimal choice, in the case of LCVaR, and when \(\kappa\) is smaller and produces a sharper distribution, in the case of LHCVaR.

\section{Discussion}
\label{sec:Discussion}

In this work, we have studied the effect of optimizing classifiers with respect to different weightings and developed robust risk measures that minimizes worst case weighted risk across a set of weightings. We subsequently show that optimizing with respect to LCVaR and LHCVaR empirically improves the worst class risk, at a reasonable cost to accuracy. One future direction for research  is to understand the Bayes optimal classifier under LCVaR and LHCVaR. Another more applied direction could be to consider domain shift.
If we formalize each prior over the classes as a weighting, optimizing LCVaR or LHCVaR may improve performance when the test class priors are different from the training class priors.

% flatex input end: [tex/main.tex]

%---------------------------------------------%
	
	\bibliography{extreme_refs}
	\pagebreak
	
	\appendix
	% flatex input: [tex/appendix.tex]
\section{Organization}
\label{sec:Organization}

Our appendices contain proofs, all of which are omitted from the main text, and additional details on the weighting approach to imbalanced classification.
In Appendix~\ref{sec:PluginClassificationDetails}, we prove our results for plug-in classification.
Additionally, we show that a threshold-shifted version of Tsybakov's noise condition implies precise rates for the convergence of expected excess risk.
Finally, we briefly discuss the universality of weighting, i.e., the fact that choosing the correct weighting is often the means to optimizing other classification metrics, for a class of classification metrics.

In Appendix~\ref{sec:TradeoffERM}, we show a result analogous to Proposition~\ref{prop:ExcessqprimeRisk} for empirical risk minimization.
However, the result is less illuminating, since it depends on the optimal classifiers \(f^{*}_{q}\) \(f^{*}_{q'}\) for weights \(q\) and \(q'\) within the class \(\functions\), which is difficult to analyze more precisely in any generality.

In Appendix~\ref{sec:RobustWeightingProofs}, we prove our results for robust weighting. 
This includes both the convergence and duality results.
In Appendix~\ref{sec:ConditionalSamplingModel}, we prove the analog of Theorem~\ref{theorem:MainStandardSampling} for the conditional sampling model. 
The only difference to observe is that the bounded differences inequality is used with respect to a different number of variables, which leads to a slightly stronger bound.

In Appendix~\ref{sec:GDA}, we discuss gradient descent-ascent, which is a standard algorithm for solving robust optimization problems.
This may be used in cases where the uncertainty set \(Q\) does not lead to LCVaR or LHCVaR.
In Appendix~\ref{sec:AdditionalLemmas} and Appendix~\ref{sec:StandardLemmas}, we provide technical and standard lemmas respectively.

Finally, we include additional experiment details, and an algorithm for analytically deriving dual variables in the empirical LCVaR and LHCVaR formulations in Appendix~\ref{sec:AdditionalExp}.
%---------------------------------------------%
%---------------------------------------------%
\section{Plug-in Classification Details}
\label{sec:PluginClassificationDetails}

In this appendix, we provide additional details surrounding plug-in classification.
We first start with the proofs of results from the main text, and then we provide more concrete results based on an additional assumption of that gives us faster rates of convergence.
Finally, we provide details on the universality of weighting.

For simplicity, we assume that our density estimator \(\hat{\eta}\) is a local polynomial estimator \citep{stone_optimal_1982}, but the properties that the estimator must have for the following proofs to succeed can also be satisfied by other nonparametric estimators such as kernelized regression \citep{krzyzak_pointwise_1987}, and nearest-neighbors regression \citep{gyorfi_rate_1981}. 
%---------------------------------------------%
%---------------------------------------------%
\subsection{Proofs}
\label{subsec:PluginProofs}

\begin{proof}[Proof of Lemma~\ref{lemma:WeightedBayesClassifier}]
By the definition of the \(q\)-weighted risk and the tower property, we have
\begin{align*}
\risk_{01, q}(f)
&=
\expect[ q_{Y} \risk_{Y}(f)] \\
&=
\expect\left[
q_{0}(1 - \eta(X)) \expect[\ind\left\{f(X) = 1 \right\} | Y = 0] +
 q_{1} \eta(X) \expect[\ind\left\{f(X) = 0\right\} | Y = 1]
\right] \\
=&
\expect\left[
q_{0}(1 - \eta(X)) \ind\left\{f(X) = 1 \right\}
+
q_{1} \eta(X) \ind\left\{f(X) = 0\right\}
\right].
\end{align*}
By inspection, we observe that the \(f^{*}\) minimizing the \(q\)-risk satisfies
\[
f^{*}(x)
=
\begin{cases}
1 & q_{0} (1 - \eta(x)) < q_{1} \eta(x) \\
0 & q_{0} (1 - \eta(x)) > q_{1} \eta(x).
\end{cases}
\]
When \(q_{0} (1 - \eta(x)) = q_{1} \eta(x)\), we note that the decision may be arbitrary because it does not affect the risk.
So, by simple algebraic manipulation, we have
\[
f^{*}(x)
=
\ind\left\{
\eta(x) \geq \frac{q_{0}}{q_{0} + q_{1}}
\right\},
\]
which completes the proof.
\end{proof}

%---------------------------------------------%
%---------------------------------------------%

Now, we turn to Proposition~\ref{prop:ExcessqRisk}, Proposition~\ref{prop:IrreducibleError}, and Proposition~\ref{prop:ExcessqprimeRisk}.
Our proofs rely on the following lemma of \cite{yang1999minimax}.
First, we introduce a few additional definitions.
Denote the \(\epsilon\)-entropy of \(\Sigma\) with respect to the \(L_p\) norm for \ \(1 \leq p \leq \infty\) by \(\metricentropy(\epsilon, \Sigma, L_p)\).
We define the norm
\begin{align*}
\norm{\hat{\eta} - \eta}_{L_1(P_X)} = \int \left|\eta(x) - \hat{\eta}(x)\right| d P_X 
\end{align*}.

\begin{lemma}[Theorem 1 of \citealt{yang1999minimax}]
Let \(\eta\) be an element of  \(\Sigma\) where \(\Sigma\) is a class of functions from \(\reals^d\) to \([0, 1]\).
Suppose the \(\epsilon\)-entropy satisfies
\begin{align*}
	\metricentropy(\epsilon, \Sigma, L_p) \leq C\epsilon^{-\rho},
\end{align*}
where \(C > 0, \rho > 0\). 
Then the minimax upper bound on the mean convergence rate of any regression estimator \(\hat{\eta}\) is 
\begin{align*}
	\underset{\hat{\eta}}{\min} \;\underset{\eta \in \Sigma}{\max}\expect\left[\norm{\eta - \hat{\eta}}_{L_1(P_X)}\right] \leq O\left(n^{-\frac{1}{2 + \rho}}\right),
\end{align*}
where the expectation is taken over the samples for estimating \(\hat{\eta}\). 
\label{lemma:YangMinimax}
\end{lemma}

The upper bound converges at a rate of $O\left(n^{-1/(2 + \rho)}\right)$ where $\rho$ is a
smoothness parameter for $\eta$, with standard assumptions on the function class of $\eta$. For the class of $\beta$-H\"older functions, $\rho = \beta / d$, which is our setting of interest.

%---------------------------------------------%
%---------------------------------------------%
\begin{proof}[Proof of Proposition~\ref{prop:ExcessqRisk}]
We start by bounding the excess $q$-risk for a classifier $f$ by
\begin{align*}
\excess_q(f) 
&=
R_q(f) - R_q(f^*_q) \\
&= 
(q_0 + q_1)
\int \left|\eta(x) - \frac{q_0}{q_0 + q_1}\right| \ind\left\{f(x) \neq f_q^*(x)\right\} d P_{X}\\
&\leq 
(q_0 + q_1)
\int \left|\eta(x) - \hat{\eta}(x)\right|  d P_X,
\end{align*}
where the upper bound follows when
\(|\eta(x) - q_{0}/(q_{0} + q_{1})| \leq |\eta(x) - \hat{\eta}(x)|\) when \(f(x) \neq f^{*}_{q}(x)\).
Finally, applying Lemma~\ref{lemma:YangMinimax} for \(\beta\)-H\"older functions as noted above completes the proof.
\end{proof}

\begin{proof}[Proof of Proposition~\ref{prop:IrreducibleError}]
The proposition follows from basic algebraic manipulations and one common observation in nonparametric classification.
We have
\begin{align*}
	\irreducible 
	&=
	\expect\left[R_{q'}(f_{q'}^*(X)) - R_{q'}(f_{q}^*(X))\right] \\
	&= 
	\int\limits \left|q'_0(1 - \eta(x)) + q'_1\eta(x)\right| dP_X \ind\left\{f_{q'}^*(x) \neq f_q^*(x)\right\}\\
	&=
	(q'_0 + q'_1)\int\limits \left|\eta(x) - t_{q'}\right| \ind\left\{f^*_{q'}(x) \neq f^*_q(x)\right\} dP_X\\
	&\leq 
	(q'_0 + q'_1)\left|t_q - t_{q'}\right| \prob\left(f_{q'}^*(X) \neq f_q^*(X)\right),
\end{align*}
where in the inequality we use the fact that if \(f_{q'}^*(X) \neq f_q^*(X)\) then \(\eta(X)\) must be in \([\underline{t}_{q, q'}, \overline{t}_{q, q'}]\).
Thus, we have \(\left|\eta(x) - t_{q'}\right| \leq |\overline{t}_{q, q'} - \underline{t}_{q, q'}| = \left|t_q - t_{q'}\right|.\)
\end{proof}

%---------------------------------------------%
%---------------------------------------------%
\begin{proof}[Proof of Proposition~\ref{prop:ExcessqprimeRisk}]
Recall that the expected estimation error is
\[
\estimation = \expect \left[R_{q'}(\hat{f}_q) - R_{q'}(f^*_q)\right]
\]
We can upper bound the term inside the expectation by
\begin{align*}
R_q'(\hat{f}_q) - R_{q'}(f^*_{q})
&=
\int\limits q'_0(1 - \eta(x))\ind\{\hat{f}_q(x) = 1\} + q'_1\eta(x) \ind\{\hat{f}_q(x) = 0\} dP_X\\
&\qquad - 
\int\limits q'_0(1 - \eta(x))\ind\{f^*_q(x) = 1\} + q'_1\eta(x) \ind\{f^*_q(x) = 0\} dP_X 
%\text{By \cref{def:ClassWeightedRisk}}
\\
&=
\int\limits (q'_0(1 - \eta(x)) - q'_1\eta(x))\ind\{\hat{f}_q(x) = 1, f^*_q(x) = 0\} dP_{X}\\
&\qquad + 
(q'_1\eta(x) - q'_0(1 -\eta(x))) 
\int \ind\{\hat{f}_q(x) = 0, f^*_q(x) = 1\} dP_X
\\
&= 
(q'_0 + q_1')\int\limits \left|\eta(x) - \frac{q_0'}{q_0' + q_1'}\right|\ind\{\hat{f}_q(x) \neq f^*_{q}(x)\} dP_X\\
&\leq 
(q_0' + q_1')\int\limits \left(\left|\eta(x) - t_{q}\right| + \left|t_{q'} - t_{q}\right|\right) \ind\{\hat{f}_q(x) \neq f^*_{q}(x)\} dP_X,
\end{align*}
where we use the triangle inequality in the final line.
Next, using the fact that \(|\eta(x) - t_{q}| \leq |\eta(x) - \hat{\eta}(x)|\) when \(f(x) \neq f^{*}_{q}(x)\), we have
\begin{align*}
R_q'(\hat{f}_q) - R_{q'}(f^*_{q})
&\leq 
\ (q_0' + q_1') \left( \int\limits \left|\eta(x) - t_q\right|\ind\left\{\hat{f}_q(x) \neq f^*_q(x)\right\} dP_X 
\right. \\ &\left. \qquad 
+ \left|t_{q'} - t_q\right|\prob\left(\hat{f}_q(x) \neq f^*_q(x)\right)\vphantom{\int\limits}\right)
\\
&\leq 
(q_0' + q_1' )\left(\int\limits \left|\eta(x) - \hat{\eta}(x)\right| dP_X + \left|t_{q'} - t_q\right|\prob\left(\hat{f}_q(x) \neq f^*_q(x)\right)\right) 
\end{align*}
Thus, we obtain the upper bound
\begin{align*}
	\estimation 
	\leq 
	(q_0' + q_1' )\left(\expect\left[\int\limits \left|\eta(x) - \hat{\eta}(x)\right| dP_X\right] + \left|t_{q'} - t_q\right|\expect\left[\prob\left(\hat{f}_q(x) \neq f^*_q(x)\right)\right]\right)
\end{align*}
Therefore we have completed the proof. Applying Lemma~\ref{lemma:YangMinimax} to the first term also proves Corollary \ref{corollary:ExcessqprimeRiskHolder}.
\end{proof}

%---------------------------------------------%
%---------------------------------------------%
\subsection{Shifted Margin Assumption}
\label{subsec:ShiftedMarginAssumption}

An important tool in nonparametric classification is the Tsybakov margin condition.
\begin{definition}
A distribution \(P_{X, Y}\) satisfies the \((\alpha, C)\)-margin condition if for all \(t > 0\), we have 
\[
\prob\left(0 \leq \left| \eta(X) - \frac{1}{2}\right|
\leq t
\right)
\leq 
C t^{\alpha}.
\]
\label{def:TsybakovMargin}
\end{definition}
Subsequent works \citep{audibert_fast_2007, chaudhuri2014rates} leverage this assumption to provide fast, explicit rates of convergence for expected risk. 
The margin condition is naturally suited to standard plug-in classification because the decision threshold is \(1/2\); for weighted plug-in classification, we need a shifted margin condition.
\begin{definition}
A distribution \(P_{X, Y}\) satisfies the \((q, \alpha, C)\)-margin condition if for all \(t > 0\), we have
\begin{align*}
    \prob\left(0 \leq \left|\eta(x) - t_{q}\right| \leq t\right) \leq C t^\alpha.   
\end{align*}
\label{def:ShiftedMarginCondition}
\end{definition}
Using the shifted margin condition, we can obtain better results than we presented in the main paper.
However, the shifted margin condition may be be less interpretable than the original margin condition.
Intuitively, the original margin condition says that there is very little probability mass where distinguishing between \(Y = 0\) and \(Y = 1\) is difficult, i.e., near \(\eta(X) = 1/2\).
For other \(t_q\), the decision may not be difficult in that \(t_q\) may be far from \(1/2\), but we would still require little mass near this point.

%---------------------------------------------%
%---------------------------------------------%

\begin{proposition}
Suppose the distribution \(P_{X, Y}\) satisfies the \((q, \alpha, C)\)-margin condition and \(X\) has a density that is lower bounded by some constant \(\mu_{\min}\) on its support.
Additionally, suppose that \(\eta\) is \(\beta\)-H\"{o}lder.
Then, the excess expected \(q'\)-risk of \(\hat{f}_q\) 
satisfies the bound
\[
\expect \excess_{q'}(\hat{f}_{q})
\leq 
(q_0' + q_1')\left(O\left(\frac{\log n}{n}\right)^{\frac{\beta}{2\beta + d}} + \left|t_{q'} - t_q\right|O\left(\frac{\log n}{n}\right)^{\frac{\alpha \beta}{2\beta + d}}\right)
+
\irreducible
\]
\label{prop:ExcessqRiskShiftedMargin}
\end{proposition}

%---------------------------------------------%
%---------------------------------------------%
Before proving this proposition, we prove a helpful lemma that leverages the shifted margin condition, similar to one from \cite{audibert_fast_2007}.
\begin{lemma}
For a fixed density estimate \(\hat{\eta}\), if \(P_{X, Y}\) satisfies the \((q, \alpha, C)\)-margin condition, then following upper bound is always true:
\begin{align*}
    \prob\left(\hat{f}_q(x) \neq f^*_q(x), \eta(x) \neq t_{q}\right)
    &\leq 
    C \left\|\eta - \hat{\eta}\right\|_\infty^\alpha.
\end{align*}
%\textcolor{red}{The left hand side is a number, i.e., non-random, but the right-hand side is a random quantity depending on the data because \(\hat{\eta}\) depends on the data. Is this bound almost sure? If it is, it should be stated that way.}
\label{lemma:ShiftedMarginBound}
\end{lemma}

\begin{proof}
We use a simple upper bound on the error probability event and apply the margin condition to obtain
\begin{align*}
\prob\left(\hat{f}_q(x) \neq f^*_q(x), \eta(x) \neq t_{q})\right) 
&\leq 
\prob\left(0 \leq \left|\eta(x) - t_{q}\right| \leq \left|\eta(x) - \hat{\eta}(x) \right|\right)\\
&\leq 
\prob\left(0 \leq \left|\eta(x) - t_{q}\right| \leq \left\|\eta - \hat{\eta}\right\|_\infty\right)\\
&\leq 
C_0 \left\| \eta - \hat{\eta}\right\|_\infty^\alpha.
    \end{align*}
This completes the proof.
\end{proof}

Since, by \cref{lemma:ShiftedMarginBound}, we have proved an upper bound in terms of \(\norm{\eta - \hat{\eta}}_\infty^\alpha\), we now cite an upper bound on that quantity that is a property of regression estimator.
%---------------------------------------------%
%---------------------------------------------%
\begin{lemma}[Theorem 1 of \citealt{stone_optimal_1982}]
	Let \(\hat{\eta}\) be a local polynomial regression estimator, and suppose \(X\) has a density that is lower bounded by some constant \(\mu_{\min} > 0\) on its support. Then, we have the following upper bound:
	\begin{align}
	\expect\left[\norm{\eta - \hat{\eta}}^\alpha_\infty\right]
	\leq 
	C \left(\frac{\log n}{n}\right)^{\frac{\alpha \beta}{2\beta + d}}.
	\end{align}
	\label{lemma:NonparamConvRate}
\end{lemma}

The above bound is the optimal rate of uniform convergence for nonparametric estimators under the regularity conditions shown here, and local polynomial regression achieves this optimal rate \citep{stone_optimal_1982}.

%---------------------------------------------%
%---------------------------------------------%
\begin{proof}[Proof of Proposition~\ref{prop:ExcessqRiskShiftedMargin}]
It suffices to prove an upper bound on the estimation error.
We have
\begin{align*}
\estimation
&\leq(q_0' + q_1')\left(
\expect\left[\int \left|\eta(x) - \hat{\eta}(x)\right| d P_X\right]
+ 
\left|t_{q'} -t_q\right| \expect\left[\prob(\hat{f}_q(x) \neq f^*_q(x))\right]\right)
\end{align*}
by the final equation of the proof of Proposition~\ref{prop:ExcessqprimeRisk}.
Next, we use the fact that for all \(x\) in \(\xspace\) we have \(\eta(x) - \hat{\eta}(x) \leq \norm{\eta -\hat{\eta}}_\infty\) and \cref{lemma:ShiftedMarginBound} to obtain
\begin{align*}
\estimation
&\leq 
(q_0' + q_1')\left(\expect\left[\left\|\eta - \hat{\eta}\right\|_\infty\right] +  \left|t_{q'} - t_q\right|C_0\expect\left[\norm{\eta - \hat{\eta}}_\infty^\alpha\right]\right)
\end{align*}
Finally, we apply \cref{lemma:NonparamConvRate} to obtain
\begin{align*}
\estimation
&\leq 
(q_0' + q_1')\left(C \left(\frac{\log n}{n}\right)^{\frac{\beta}{2\beta + d}} + \left|t_{q'} - t_{q}\right|C_0C \left(\frac{\log n}{n}\right)^{\frac{\alpha \beta}{2\beta + d}}\right),
%&\leq (q_0' + q_1')\left(O\left(\frac{\log n}{n}\right)^{\frac{\beta}{2\beta + d}} + \left|t_{q'} - t_q\right|O\left(\frac{\log n}{n}\right)^{\frac{\alpha \beta}{2\beta + d}}\right),
\end{align*}
which completes the proof.
\end{proof}

%---------------------------------------------%
%---------------------------------------------%
\subsection{Universality of Weighting}
\label{subsec:UniversalityOfWeighting}

Since we may be interested in performance in error metrics other than risk, we discuss other classification metrics here.
In particular, we simply show that weighting is ``universal'' in that it can be used to optimize these other classification metrics.
The reason for this is that, in plug-in classification, optimizing many classification metrics is equivalent to altering the threshold for the classification, and this has been observed to lead to the optimal decision rule in many cases \citep{lewis1995evaluating, menon2013statistical, narasimhan2014statistical, koyejo2014consistent}.
We examine the specific case of metrics considered in \cite{koyejo2014consistent}.
%---------------------------------------------%
%---------------------------------------------%
\begin{definition}
Let $f$ be a classifier over \(\xspace\). 
Define the true positive, false negative, false positive, and true negative proportions to be
	\begin{align*}
		\TP &= \prob(Y = 1, f(X) = 1) &\qquad \FP &= \prob(Y = 0, f(X) = 1)\\
		\FN &= \prob(Y = 1, f(X) = 0) &\qquad \TN &= \prob(Y = 0, f(X) = 0).
	\end{align*}
A linear-fractional metric is defined as 
	\begin{align*}
		\loss(f, P_X, \eta) = \frac{a_0 + a_{11}\TP + a_{10}\FP + a_{01}\FN + a_{00}\TN}{b_0 + b_{11}\TP + b_{10}\FP + b_{01}\FN + b_{00}\TN}
	\end{align*} 
for constants $a_0, a_{11}, a_{10}, a_{01}, a_{00}, b_0, b_{11}, b_{10}, b_{01}, b_{00}$.
\label{def:LinearFrac}
\end{definition}
%---------------------------------------------%
%---------------------------------------------%
\cite{koyejo2014consistent} showed that the optimal classifier for any linear-fractional metric is simply a threshold classifier. 
Specifically, the following theorem is true.
\begin{theorem}[\citealt{koyejo2014consistent}]
	Let \(\loss\) be a linear-fractional metric, and let 
	\(P_X\) be absolutely continuous with respect to the dominating measure \(\numeasure\) on \(\xspace\).  
	Define
	\begin{align*}
		\loss^* = \underset{f}{\text{\normalfont max }} \loss(f, P_X, \eta)
	\end{align*}
	and
	\begin{align*}
		\threshold^* = \frac{(b_{10} - b_{00})\loss^* - a_{10} + a_{00}}{a_{11} - a_{10} - a_{01} + a_{00} - (b_{11} - b_{10} - b_{01} + b_{00})\loss^*}.
	\end{align*}
Then, the optimal classifier for $\loss$ is 
\(
f_\loss^*(x) = \ind\left\{\eta(x) > \delta^* \right\}
\) 
if 
\[
a_{11} - a_{10} - a_{01} + a_{00} - (b_{11} - b_{10} - b_{01} + b_{00})\loss^* > 0
\] and
\(
f_\loss^*(x) = \ind\left\{\eta(x) < \delta^* \right\}
\)
otherwise.
\end{theorem}

\begin{corollary}
We note by Proposition~\ref{prop:ExcessqRisk} that for an metric $\loss$ where 
\[
a_{11} - a_{10} - a_{01} + a_{00} - (b_{11} - b_{10} - b_{01} + b_{00})\loss^* > 0,
\]
if we set define \(q\) to be 
	\begin{align*}
		q_0 &= (b_{10} - b_{00})\loss^* - a_{10} + a_{00}\\
		q_1 &= (b_{01}-  b_{11})\loss^* - a_{01} + a_{11},
	\end{align*}
then $f_q^* = f_\loss^*$. 
\label{cor:WeightUniversality}
\end{corollary}

Performance metrics that are used in evaluating classifiers such as F1 and arithmetic mean satisfy the the conditions of Corollary~\ref{cor:WeightUniversality}. 
Thus, we can reformulate optimization of a classifier in these error metrics as a specific weighting the risk.

%---------------------------------------------%
%---------------------------------------------%
\section{The Fundamental Trade-off in Empirical Risk Minimization}
\label{sec:TradeoffERM}

Part of our motivation for the robust weighted problem is the fundamental trade-off under different weightings \(q\) and \(q'\).
We demonstrated this for plug-in classification in the main text because it elucidates the nature of the problem naturally via thresholds,
but we should also convince ourselves that this is not simply a quirk of plug-in classification.
To this end, we provide a brief analysis for empirical risk minimization.

Let \(\fhat_{q}\) and \(f_{q}^{*}\) denote the empirical risk minimizer and risk minimizer within \(\functions\).
Define the excess risk to be the difference between \(\risk(\fhat_{q})\) and \(\risk_{q}(f_{q}^{*})\).
Suppose that we have a uniform convergence guarantee
\[
\risk_{q}(f) - \hat{\risk}_{q}(f)
\leq 
O\left(n^{-\frac{1}{2}}\right)
\]
for all \(f\) in \(\functions\).
Then, a standard chaining argument reveals that the excess risk decay rate satisfies
\begin{align*}
& \begin{aligned}
\excess_{q}(\fhat_q)
&=
\risk_{q}(\fhat_{q}) - \risk(f_{q}^{*}) \\
&=
\risk_{q}(\fhat_{q}) - \hat{\risk}_{q}(\fhat_{q})
+ \hat{\risk}_{q}(\fhat_{q}) - \hat{\risk}_{q}(f_{q}^{*})
+ \hat{\risk}_{q}(f_{q}^{*}) - \risk(f_{q}^{*}) \\
&\leq 
O\left(n^{-\frac{1}{2}}\right) + 0 + O\left(n^{-\frac{1}{2}}\right) \\
&= 
O\left(n^{-\frac{1}{2}}\right),
\end{aligned}
\end{align*}
where in the inequality we used our uniform convergence guarantee twice and the fact that \(\fhat_q\) is the empirical \(q\)-risk minimizer.
This mirrors the case of \(q\)-weighted plug-in estimation in that the excess \(q\)-risk still converges to \(0\) at the standard rate.

On the other hand, we obtain a constant term when performing a similar analysis for \(\excess_{q'}(\fhat_{q})\).
Specifically, we get
\begin{align*}
& \begin{aligned}
\excess_{q'}(\fhat_{q})
&=
\risk_{q'}(\fhat_{q}) - \risk_{q'}(f_{q}^{*}) \\
&=
\risk_{q}(\fhat_{q}) - \risk_{q}(f_{q}^{*})
+ 
\risk_{q'}(\fhat_{q}) - \risk_{q}(\fhat_{q}) 
+
\risk_{q}(f_{q}^{*}) - \risk_{q'}(f_{q'}^{*}) \\
&\leq 
O\left(n^{-\frac{1}{2}}\right)
+
\risk_{q'}(\fhat_{q}) - \risk_{q}(\fhat_{q}) 
+
\risk_{q}(f_{q}^{*}) - \risk_{q'}(f_{q'}^{*}).
\end{aligned}
\end{align*}
Now, using the prior convergence result for the empirical risk minimizers, we obtain
\begin{align*}
& \begin{aligned}
\excess_{q'}(\fhat_{q})
&\leq 
\risk_{q'}(\fhat_{q}) - \risk_{q}(\fhat_{q}) 
+
\risk_{q}(f_{q}^{*}) - \risk_{q'}(f_{q'}^{*}) 
+
O\left(n^{-\frac{1}{2}}\right)\\
&\leq 
\risk_{q'}(f_{q}^{*}) - \risk_{q}(f_{q}^{*})
+
\risk_{q}(f_{q}^{*}) - \risk_{q'}(f_{q'}^{*}) 
+
O\left(n^{-\frac{1}{2}}\right)\\
&=
\underbrace{\risk_{q'}(f_{q}^{*}) - \risk_{q'}(f_{q'}^{*}) }_{A}
+
O\left(n^{-\frac{1}{2}}\right).
\end{aligned}
\end{align*}
Since \(f_{q'}^{*}\) minimizes \(\risk_{q'}\) and \(f_{q}^{*}\) minimizes \(\risk_{q}\), we see that \(A \geq 0\).
Thus, even though there is not a clear threshold interpretation, we do see that there is irreducible error that arises in the empirical risk minimization setting as well.

%---------------------------------------------%
%---------------------------------------------%
\section{Robust Weighting Proofs}
\label{sec:RobustWeightingProofs}
In this section, we prove our results for robust weighting.
We start with our generalization and excess risk bounds.

\begin{proof}[Proof of Theorem~\ref{theorem:MainStandardSampling}]
Define the risk \(\risk_{i, \ind}\) as
\begin{align*}
& \begin{aligned}
\hat{\risk}_{i, \ind}(f)
&=
\hat{p}_{i} \hat{\risk}_{i}(f)
=
\frac{1}{n}
\sum_{j = 1}^{n} \marginloss(f, z_{j}) \ind\left\{y_{j} = i\right\}.
\end{aligned}
\end{align*}
Let \( \risk_{i, \ind}(f)\) denote \(\expect \hat{\risk}_{i, \ind}(f)\).
Note that we have 
\[
\risk_{i, \ind}(f) 
=
\frac{1}{n} \sum_{j = 1}^{n} \expect \left[\marginloss(f, z_{j}) \ind\left\{y_{j} = i\right\}\right]
=
\frac{1}{n} \sum_{j = 1}^{n} p_{i} \expect\left[\marginloss(f, z_{j}) | y_{j} = i\right]
=
p_{i} \risk_{i}(f).
\]
By definition, we have
\begin{align*}
& \begin{aligned}
\risk_{Q}(f)
=
\sup_{q \in Q} \sum_{i = 1}^{k} q_{i} p_{i} \risk_{i}(f)
=
\sup_{q \in Q} \sum_{i = 1}^{k} q_{i} \risk_{i, \ind}(f),
\end{aligned}
\end{align*}
and so for our purposes, it suffices to analyze \(\hat{\risk}_{i, \ind}\).
Define the class
\[
\functions_{i, \ind}
=
\left\{\marginloss(f, \cdot) \ind\left\{y_{j} = i\right\}: f \in \functions\right\}.
\]
By Lemma~\ref{lemma:StandardRademacher}, we have with probability at least \(1 - \delta / k\) that
\begin{align*}
& \begin{aligned}
\risk_{i, \ind}(f)
\leq 
\hat{\risk}_{i, \ind}(f)
+
2 \rademacher_{n}(\marginloss \circ \functions_{i, \ind})
+
\sqrt{\frac{\log \frac{k}{\delta}}{2n}}
\end{aligned}
\end{align*}
for each \(f\) in \(\functions\).
So, it suffices to analyze the Rademacher complexity term.
Let \(\sigma_{j}\) be iid Rademacher random variables.
We condition on the value of \(y_{1}, \ldots, y_{n}\).
Let \(\mathcal{H}_{Y}\) be the sigma-field \(\sigma(y_{1}, \ldots, y_{n})\).
Suppose without loss of generality that under the conditioning, we have \(y_{1} = \cdots = y_{N_{i}} = i\) and \(y_{j} \neq i\) for all \(j > N_{i}\).
Then, we have
\begin{align*}
& \begin{aligned}
\rademacher_{n}(\functions_{i, \ind})
&=
\frac{1}{n}
\expect
\expect\left[
\sup_{f \in \functions}
\sum_{i = 1}^{n} 
\sigma_{j}
\marginloss(f, z_{j}) \ind\{y_{j} = i\}
\biggr| \mathcal{H}_{Y}
\right] \\
&=
\frac{1}{n}
\expect
\expect\left[
\sup_{f \in \functions}
\sum_{j = 1}^{N_{i}}
\sigma_{j}
\marginloss(f, z_{j})
\biggr|
\mathcal{H}_{Y}
\right] \\
&=
\expect
\left[
\frac{N_{i}}{n}
\hat{\rademacher}_{N_{i}}(\marginloss \circ \functions )
\right].
\end{aligned}
\end{align*}

By the proof of Lemma~\ref{lemma:KuznetsovRademacher}, we have
\begin{align*}
& \begin{aligned}
\hat{\rademacher}_{N_{i}}(\marginloss \circ \functions )
\leq 
2 k \hat{\rademacher}_{N_{i}}(\Pi_{1}(\functions)).
\end{aligned}
\end{align*}
Putting everything together completes the proof of the generalization bound; now we turn to the excess \((\functions, q)\)-risk bound.

Recall that \(\hat{f}_{Q}\) is the empirical \(Q\)-risk minimizer and \(f^{*}_{Q}\) is the population \(Q\)-risk minimizer.
By Lemma~\ref{lemma:ModifiedRademacher}, we have with probability at least \(1 - \delta / k\) that
\begin{align*}
\risk_{i, \ind}(\hat{f}_{Q}) 
&\leq 
\hat{\risk}_{i, \ind}(\hat{f}_{Q}) 
+
4 \rademacher_{n}(\marginloss \circ \functions_{i, \ind}) + \sqrt{\frac{\log \frac{k}{\delta}}{2 n}}.
\end{align*}
Summing, we have
\begin{align*}
\risk_{q}(\hat{f}_{Q})
=
\sum_{i = 1}^{k} q_{i} p_{i} \risk_{i}(\hat{f}_{Q})
&=
\sum_{i = 1}^{k} q_{i} (\risk_{i, \ind}\hat{f}_{Q})  \\
&\leq 
\sum_{i = 1}^{k} q_{i}\left(
\hat{\risk}_{i, \ind}(\hat{f}_{Q}) 
+
4 \rademacher_{n}(\marginloss \circ \functions_{i, \ind}) + \sqrt{\frac{\log \frac{k}{\delta}}{2 n}}
\right) \\
&\leq 
\hat{\risk}_{q}(\hat{f}_{Q})
+
\sum_{i = 1}^{k} q_{i}p_{i} 
\left(
\frac{4}{p_{i}}\rademacher_{n}(\marginloss \circ \functions_{i, \ind}) +  \sqrt{\frac{\log \frac{k}{\delta}}{2 p_{i}^{2} n}} 
\right).
\end{align*}
Using the proof of Lemma~\ref{lemma:KuznetsovRademacher} as before, we then obtain
\[
\risk_{q}(\hat{f}_{Q})
\leq 
\hat{\risk}_{q}(\hat{f}_{Q})
+
\sum_{i = 1}^{k} q_{i}p_{i} 
\left(
8k \expect\left[\frac{N_{i}}{p_{i} n}\emprad_{N_{i}}(\Pi_{1}(\functions))\right] +  \sqrt{\frac{\log \frac{k}{\delta}}{2 p_{i}^{2} n}} 
\right).
\]
Thus, by taking supremums, we observe that
\begin{equation}
\risk_{Q}(\hat{f}_{Q}) 
\leq
\emprisk_{Q}(\hat{f}_Q)
+
\sup_{q \in Q}
\sum_{i = 1}^{k} q_{i}p_{i} 
\left(
8k \expect\left[\frac{N_{i}}{p_{i} n}\emprad_{N_{i}}(\Pi_{1}(\functions))\right] +  \sqrt{\frac{\log \frac{k}{\delta}}{2 p_{i}^{2} n}} 
\right).
\label{eqn:ExcessFRiskProofEqn1}
\end{equation}
Similarly, by Lemma~\ref{lemma:ModifiedRademacher}, we have
\[
-\risk_{i, \ind}(f^{*}_{Q}) 
\leq 
-\emprisk_{i, \ind}(f^{*}_{Q}) 
+ 
4 \rademacher_{n}(\marginloss \circ \functions_{i, \ind}) + \sqrt{\frac{\log \frac{k}{\delta}}{2 n}}.
\]
Summing as before and using the proof of Lemma~\ref{lemma:KuznetsovRademacher}, we have
\begin{align*}
-\risk_{q}(f^{*}_{Q})
\leq 
-\emprisk_{q}(f^{*}_{Q})
+
\sum_{i = 1}^{k} q_{i}p_{i} 
\left(
8k \expect\left[\frac{N_{i}}{p_{i} n}\emprad_{N_{i}}(\Pi_{1}(\functions))\right] +  \sqrt{\frac{\log \frac{k}{\delta}}{2 p_{i}^{2} n}} 
\right).
\end{align*}
Taking the infimum and using Lemma~\ref{lemma:SimpleInfimum}, we have
\begin{equation}
-\risk_{Q}(f^{*}_{Q})
\leq 
-\emprisk_{Q}(f^{*}_{Q})
+
\sup_{q \in Q}
\sum_{i = 1}^{k} q_{i}p_{i} 
\left(
8k \expect\left[\frac{N_{i}}{p_{i} n}\emprad_{N_{i}}(\Pi_{1}(\functions))\right] +  \sqrt{\frac{\log \frac{k}{\delta}}{2 p_{i}^{2} n}} 
\right).
\label{eqn:ExcessFRiskProofEqn2}
\end{equation}
Summing equation~\eqref{eqn:ExcessFRiskProofEqn1} and equation~\eqref{eqn:ExcessFRiskProofEqn2} and noting that \(\hat{f}_{Q}\) minimizes the empirical robust risk, we have
\[
\excess_{Q}(\functions)
=
\risk_{Q}(\hat{f}_{Q}) - \risk_{Q}(f^{*}_{Q})
\leq 
2\sup_{q \in Q}
\sum_{i = 1}^{k} q_{i}p_{i} 
\left(
8k \expect\left[\frac{N_{i}}{p_{i} n}\emprad_{N_{i}}(\Pi_{1}(\functions))\right] +  \sqrt{\frac{\log \frac{k}{\delta}}{2 p_{i}^{2} n}} 
\right),
\]
and this completes the proof.
\end{proof}

%---------------------------------------------%
%---------------------------------------------%

\begin{proof}[Proof of Corollary~\ref{cor:MainStandardSampling}]
The only thing we need to do here is calculate the Rademacher complexity term of Theorem~\ref{theorem:MainStandardSampling}.
Using our assumption and Jensen's inequality, we have
\begin{align*}
&\begin{aligned}
\expect\left[\frac{N_{i}}{n} \hat{\rademacher}_{N_{i}}(\Pi_{1}(\functions))\right]
&\leq 
\frac{C(\functions)}{n}
\expect\left[
\sqrt{N_{i}}
\right] 
\leq 
\frac{C(\functions)}{n}
\expect[N_{i}]^{1/2} 
=
C(\functions) \sqrt{\frac{p_{i}}{n}}.
\end{aligned}
\end{align*}
This completes the proof of the corollary.
\end{proof}

%---------------------------------------------%
%---------------------------------------------%
Next, we prove our duality results. 
We start with LCVaR.
%---------------------------------------------%
%---------------------------------------------%
\begin{proof}[Proof of Proposition~\ref{prop:CVaRDual}]
The Lagrangian of LCVaR is
\begin{align*}
& \begin{aligned}
L(q, \lambda)
&=
\expect[ q_{Y} R_{Y}(f) ]
+
\lambda(1 - \expect[ q_{Y}])
=
\expect\left[
q_{Y} (R_{Y}(f) - \lambda)\right]
+
\lambda.
\end{aligned}
\end{align*}

Our goal is to use the minimax theorem, which we state as Theorem~\ref{theorem:MinimaxTheorem}, to switch the infimum over \(\lambda\) and the supremum over \(q\).
First, we do not need the minimax theorem to obtain
\begin{align}
& \begin{aligned}
\inf_{\lambda\in \reals} L(q, \lambda)
&\leq 
\inf_{\lambda \in \reals}
\sup_{q: q(\cdot) \in [0, \alpha^{-1}]}
\expect\left[
q_{Y} (\risk_{Y}(f) - \lambda)\right]
+
\lambda 
=
\inf_{\lambda \in \reals}
\left\{\expect\left[
\alpha^{-1} \expect(\risk_{Y}(f) - \lambda)_{+}
+
\lambda
\right]
\right\},
\label{eqn:DualForLambda}
\end{aligned}
\end{align}
since the inequality follows the trivial direction of the minimax theorem and we can solve the inner maximization problem by setting 
\[
q_{i}
=
\begin{cases}
0 & \risk_{i}(f) - \lambda < 0 \\
\alpha^{-1} & \risk_{i}(f) - \lambda \geq 0.
\end{cases}
\]

Our present goal is to verify the conditions of the minimax theorem.
First, we note that \(\lambda \mapsto L(q, \lambda)\) is linear and therefore convex for any \(q\), and similarly, \(q \mapsto L(q, \lambda)\) is linear and therefore concave for any \(q\).
Additionally, the domain of \(q\), in this case \([0, \alpha^{-1}]^{k}\), is compact and convex by definition; so we only need to prove that it suffices to consider \(\lambda\) on a compact, convex domain.

Denote the right hand side of equation~\eqref{eqn:DualForLambda} by \(\inf_{\lambda \in \reals} D(\lambda)\).
Let \(F_{f}(\lambda)\) denote the cumulative distribution function of \(\risk_{Y}\) at \(\lambda\).
By Lemma~\ref{lemma:CVaRDerivative}, the derivative of \(D(\lambda)\) is given by
\[
D'(\lambda) 
=
1 + \alpha^{-1}(F_{f}(\lambda) - 1),
\]
when \(F_{f}\) is continuous at \(\lambda\).
If it is not, then the same result holds for the left and right limits.
Thus by considering signs of the derivative, we see that \(\lambda\) achieves minimizes \(D(\lambda)\) for a value in the interval \([\lambda_{*}(f), \lambda^{*}(f)]\) where
\[
\lambda_{*}(f) = \inf\{t : F_{f}(t) \geq 1 - \alpha\}
\text{ and }
\lambda^{*}(f) = \sup\{t : F_{f}(t) \leq 1 - \alpha\}.
\]
Note further that when \(\functions\) is compact in, say, sup norm, then we also have finite \(\lambda_{*} = \inf_{f \in \functions} t_{*}(\lambda)\) and 
\(\lambda^{*} = \inf_{f \in \functions} t^{*}(\lambda)\).
In any case, we see that it suffices to define \(\lambda\) on a compact set \(\Lambda = [\lambda_*, \lambda^*]\), and so we may assume without loss of generality that the domain of \(\lambda\) is compact.

This verifies the conditions of the minimax theorem, and so
we have
\begin{align*}
& \begin{aligned}
\lcvar_{\alpha}(f)
&=
\inf_{\lambda \in \reals}
\sup_{q: q(\cdot) \in [0, \alpha^{-1}]}
\expect\left[
q_{Y} (\risk_{Y}(f) - \lambda)\right]
+
\lambda 
=
\inf_{\lambda \in \reals}
\left\{\expect\left[
\alpha^{-1} \expect(\risk_{Y}(f) - \lambda)_{+}
+
\lambda
\right]
\right\},
\end{aligned}
\end{align*}
which completes the proof.
\end{proof}

%---------------------------------------------%
%---------------------------------------------%
Next, we consider LHCVaR.
%---------------------------------------------%
%---------------------------------------------%
\begin{proof}[Proof of Proposition~\ref{prop:HCVaRDual}]
The proof is similar to that of Proposition~\ref{prop:CVaRDual}.
The Lagrangian of LHCVaR is 
\begin{align*}
& \begin{aligned}
L(q, \lambda)
&=
\expect[q_{Y} \risk_{Y}(f)]
+ 
\lambda \left(1 - \expect[q_{Y}]\right)
=
\expect\left[
q_{Y}
(\risk_{Y}(f) - \lambda)
\right]
+
\lambda.
\end{aligned}
\end{align*}
Next, by the trivial direction of the minimax theorem, we have
\begin{align}
& \begin{aligned}
\lhcvar_{\alpha}(f)
&\leq 
\inf_{\lambda \in \reals}
\sup_{q: q_{Y} \in [0, \alpha^{-1}_{Y}]}
L(q, \lambda) 
=
\inf_{\lambda \in \reals}
\expect\left[
\alpha^{-1}_{Y}(\risk_{Y}(f) - \lambda)_{+}
\right]
+
\lambda.
\label{eqn:HCVaRLambda}
\end{aligned}
\end{align}
So, now our goal is to verify the conditions of the minimax theorem.
As with LCVaR, the Lagrangian \(L\) is linear and therefore concave in \(q\); is linear and therefore convex in \(\lambda\); and is defined over a compact domain of values of \(q\) given by \([0, \alpha^{-1}]^{k}\).
Thus, the only difficulty, as with LCVaR, is showing that it suffices to define \(\lambda\) over a compact interval.
To this end, define the right hand side of equation~\eqref{eqn:HCVaRLambda} to be \(\inf_{\lambda \in \reals} H(\lambda)\).
It suffices to show that \(D(\lambda)\) achieves its infimum on a closed interval, in which case we can restrict the domain of \(\lambda\) to this compact, convex set.

To prove such an interval exists, we wish to show that there exist constants \(\lambda_{*}\) and \(\lambda^{*}\) such that \(H\) is decreasing for all \(\lambda < \lambda_{*}\) and increasing for all \(\lambda > \lambda^{*}\).
By Lemma~\ref{lemma:HCVaRDerivative}, we see that the derivative of \(H\) is
\[
H'(\lambda)
=
1 - \expect[\alpha_{Y}^{-1} \ind\left\{R_{Y}(f) > \lambda\right\}]
=
1 - \sum_{i = 1}^{k} \alpha_{i}^{-1} p_{i} \ind\left\{R_{i}(f) > \lambda\right\}
\]
when \(H'\) exists; otherwise the result holds for the left and right derivatives.
Let \(\lambda_{*}(f) = \min_{i = 1,\ldots, k} R_{i}(f)\).
Then, for \(\lambda \leq \lambda_{*}(f)\), we have
\[
H'(\lambda)
=
1 - \sum_{i = 1}^{k} \alpha_{i}^{-1} p_{i} \leq 0.
\]
Next, pick \(\lambda^{*}(f) = \max_{i = 1, \ldots, k} R_{i}(f) + 1\).
Then, for all \(\lambda \geq \lambda^{*}(f)\), we have
\[
H'(\lambda) 
= 
1 \geq 0.
\]
If \(\ell\) is continuous, then each \(R_{i}(f)\) is continuous in \(f\). 
Moreover, when \(\functions\) is compact on \(\xspace\) in the supremum norm,
then we can define finite constants \(\lambda_{*} = \inf_{f \in \functions} \lambda_{*}(f)\) and \(\lambda^{*}(f) = \sup_{f \in \functions} \lambda^{*}(f)\).

Thus, we may restrict the domain of \(\lambda\) to \([\lambda_{*}, \lambda^{*}]\) without loss of generality.
The minimax theorem now implies that equation~\eqref{eqn:HCVaRLambda} holds with equality, which completes the proof.
\end{proof}

%---------------------------------------------%
%---------------------------------------------%
\section{Results for the Conditional Sampling Model}
\label{sec:ConditionalSamplingModel}

Now, we present the alternative result for the conditional sampling model.
Recall that \(n_i\) is the number of samples of class \(i\), which is assumed to be fixed.

\begin{theorem}
Let \(\ell\) be the multiclass margin loss.
With probability at least \(1 - \delta\), for every \(f\) in \(\functions\) we have
\[
\risk_{Q}
\leq 
\max_{q \in Q}
\left\{
\hat{\risk}_{q}(f)
+
\sum_{i = 1}^{k} 
q_{i} \hat{p}_{i} \left(2k \rademacher_{n_{i}}(\functions) 
+ 
\sqrt{\frac{\log \frac{k}{\delta}}{2n_{i}}}
\right)
\right\}.
\]
\label{theorem:MainConditionalSampling}
\end{theorem}

%---------------------------------------------%
%---------------------------------------------%
\begin{proof}%[Proof of Theorem~\ref{theorem:MainConditionalSampling}]
The proof is similar to that of \cite{cao2019learning}.
We apply Lemma~\ref{lemma:StandardRademacher} and Lemma~\ref{lemma:KuznetsovRademacher} to obtain
\begin{align*}
&\begin{aligned}
\risk_{i}(f)
&\leq 
\hat{\risk}_{i}(f)
+
2k \rademacher_{n_{i}}(\functions) 
+
\sqrt{\frac{\log \frac{k}{\delta}}{2n_{i}}}.
\end{aligned}
\end{align*}
Multiplying by \(q_{i} \hat{p}_{i}\), summing over \(i\), and taking a supremum over \(Q\) completes the proof.
\end{proof}

%---------------------------------------------%
%---------------------------------------------%
\section{Gradient Descent-Ascent}
\label{sec:GDA}

In general, the robust classification problem is a saddle-point problem.
For our purposes, define a saddle-point problem to be an optimization problem of the form
\begin{equation}
\inf_{a \in \mathcal{A}} \sup_{b \in \mathcal{B}}
f(a, b).
\label{eqn:SaddlePoint}
\end{equation}

One of the seminal results in game theory is that the minimax problem is equivalent to the maximin problem.

\begin{theorem}[minimax theorem]
Let \(\mathcal{A}\) and \(\mathcal{B}\) be compact convex sets.
Let \(f: \mathcal{A} \times \mathcal{B} \to \reals\) be a function such that \(a \mapsto f(a, b)\) is convex and \(b \mapsto f(a, b)\) is concave.
Then, we have
\[
\inf_{a \in \mathcal{A}} \sup_{b \in \mathcal{B}}
f(a, b)
=
\sup_{b \in \mathcal{B}}
\inf_{a \in \mathcal{A}} 
f(a, b).
\]
\label{theorem:MinimaxTheorem}
\end{theorem}

\begin{algorithm2e}[!ht]
\SetKwInOut{Input}{Input}
\SetKwInOut{Output}{Output}
\Input{Convex domain \(\mathcal{A}\), \(a_{1} \in \mathcal
{A}\), step sizes \(\eta_{t}\), number of rounds \(T\)}
\For{\(t = 1, \ldots, T\)}{
    Play \(a_{t}\) and observe cost \(f_{t}(a_{t})\).
    
    Update and project 
    \begin{align*}
    &\begin{aligned}
    x_{t + 1}&= a_{t} - \eta_{t} \nabla f_{t}(a_{t}) \\
    a_{t + 1}&= \Pi_{\mathcal{A}}(x_{t + 1}).
    \end{aligned}
    \end{align*}
    }

\Output{The average iterate \(\bar{a}_{T} = \frac{1}{T} \sum_{t = 1}^{T} a_{t}\).}
\caption{Online Gradient Descent}
\label{alg:OGD}
\end{algorithm2e}

\begin{lemma}[Theorem~3.1 of \citealt{hazan2016introduction}]
Let \(f_{1}, \ldots, f_{T}: \mathcal{A} \to \reals\) be a sequence of \(L\)-Lipschitz convex functions.
If the step size for online gradient descent is chosen to be
\[
\eta_{t} = \frac{D}{L \sqrt{t}},
\]
then we have
\[
\sum_{t = 1}^{T} f_{t}(a_{t})
-
\min_{a^{*} \in \mathcal{A}} \sum_{t = 1}^{T} f_{t}(a^{*})
\leq 
\frac{3}{2}DL \sqrt{T}.
\]
\label{lemma:OGD}
\end{lemma}

Now we return to the saddle-point problem.
We give the gradient descent-ascent algorithm in Algorithm~\ref{alg:GDA} and the convergence result in Proposition~\ref{prop:GDAConvergence}.

\begin{algorithm2e}[!ht]
\SetKwInOut{Input}{Input}
\SetKwInOut{Output}{Output}
\Input{Convex-concave function \(f\), step sizes \(\eta_{a, t}\) and \(\eta_{b, t}\), number of rounds \(T\)}
\For{\(t = 1, \ldots, T\)}{
    Play \((a_{t}, b_{t})\) and observe cost \(f(a_{t}, b_{t})\).
    
    Update and project 
    \begin{align*}
    &\begin{aligned}
    x_{t + 1}&= a_{t} - \eta_{t} \nabla_{a} f(a_{t}, b_{t}) \\
    a_{t + 1}&= \Pi_{\mathcal{A}}(x_{t + 1}).
    \end{aligned}
    \end{align*}

    Update and project 
    \begin{align*}
    &\begin{aligned}
    y_{t + 1}&= b_{t} + \eta_{t} \nabla_{b} f(a_{t}, b_{t}) \\
    b_{t + 1}&= \Pi_{\mathcal{A}}(y_{t + 1}).
    \end{aligned}
    \end{align*}
    }

\Output{The average iterates \(\bar{a}_{T} = \frac{1}{T} \sum_{t = 1}^{T} a_{t}\) and \(\bar{b}_{T} = \frac{1}{T} \sum_{t = 1}^{T} b_{t}\).}
\caption{Gradient Descent-Ascent}
\label{alg:GDA}
\end{algorithm2e}

\begin{proposition}
Let \(\mathcal{A}\) and \(\mathcal{B}\) be convex, compact sets.
Suppose that \(\mathcal{A}\) has diameter \(D_{a}\) and \(\mathcal{B}\) has diameter \(D_{b}\).
Let \(f:\mathcal{A} \times \mathcal{B} \to \reals\) be convex-concave, \(L_{a}\)-Lipschitz in its first argument, and \(L_{b}\)-Lipschitz in its second argument.
Let \((a^{*}, b^{*})\) denote the solution to the saddle-point problem of equation~\eqref{eqn:SaddlePoint}.
If \((\bar{a}_{T}, \bar{b}_{T})\) is the output of Algorithm~\ref{alg:GDA}, then we have
\[
f(a^{*}, b^{*}) - \frac{3 (L_{a} D_{a} + L_{b} D_{b})}{2\sqrt{T}}
\leq 
f(\bar{a}_{T}, \bar{b}_{T})
\leq 
f(a^{*}, b^{*}) + \frac{3 (L_{a} D_{a} + L_{b} D_{b})}{2\sqrt{T}}.
\]
\label{prop:GDAConvergence}
\end{proposition}

First, we want to use a lemma from online convex optimization. 
For this, we also state the standard online gradient descent algorithm.
Here, we use \(\Pi_{\mathcal{A}}\) to denote projection onto the set \(\mathcal{A}\).

\begin{proof}
The proof is fairly straightforward from pre-existing results on online gradient descent; so we state it here.
We start first with the upper bound.
Define the ``regret'' to be
\begin{align*}
&\begin{aligned}
R_{T}
&=
\sum_{t = 1}^{T} \left[f(a_{t}, b_{t}) - f(a^{*}, b^{*})\right]
\end{aligned}
\end{align*}
where \((a^{*}, b^{*})\) is a solution to the saddle-point problem.
Then, we have the decomposition
\begin{align}
&\begin{aligned}
R_{T}
&=
\sum_{t = 1}^{T} \left[f(a_{t}, b_{t}) - f(a^{*}, b_{t})\right]
+
\sum_{t = 1}^{T} \left[f(a^{*}, b_{t}) - f(a^{*}, b^{*})\right] 
\leq 
\frac{3}{2} L_{a}D_{a} \sqrt{T} + 0, 
\label{eqn:GDAUpper}
\end{aligned}
\end{align}
where the inequality follows from applying Lemma~\ref{lemma:OGD} and noting that the second summand is nonpositive by the definition of \(b^{*}\).
Similarly, we have
\begin{align}
&\begin{aligned}
-R_{T}
&=
\sum_{t = 1}^{T} \left[f(a^{*}, b^{*}) - f(a_{t}, b_{t})\right]
+
\sum_{t = 1}^{T} \left[f(a_{t}, b_{t}) - f(a_{t}, b_{t})\right] 
\leq 
0
+
\frac{3}{2} L_{b} D_{b} \sqrt{T}.
\label{eqn:GDALower}
\end{aligned}
\end{align}

So, now we consider the averaged iterates.
We have
\begin{align*}
&\begin{aligned}
f(\bar{a}_{T}, \bar{b}_{T})
&\leq 
\max_{b \in \mathcal{B}}
f(\bar{a}_{T}, b)  \\
&\leq 
\frac{1}{T} \max_{b \in \mathcal{B}}
\sum_{t = 1}^{T} f(a_{t}, b) \\
&=
f(a^{*}, b^{*})
+ 
\frac{1}{T} 
\max_{b \in \mathcal{B}}
\sum_{t = 1}^{T}
[f(a_{t}, b) - f(a_{t}, b_{t})]
+
\frac{1}{T}\sum_{t = 1}^{T}
[f(a_{t}, b_{t}) - f(a^{*}, b^{*})] \\
&\leq 
f(a^{*}, b^{*})
+
\frac{3L_{b}D_{b}}{2\sqrt{T}} + \frac{3L_{a} D_{a}}{2\sqrt{T}}.
\end{aligned}
\end{align*}
Note that the second inequality is due to convexity, and the third is due to Lemma~\ref{lemma:OGD} and equation~\eqref{eqn:GDAUpper}.

Similarly, we have
\begin{align*}
&\begin{aligned}
f(\bar{a_{T}}, \bar{b}_{T})
&\geq 
\min_{a \in \mathcal{A}}
f(a, \bar{b}_{T}) \\
&\geq 
\frac{1}{T} 
\min_{a \in \mathcal{A}}
\sum_{t = 1}^{T} f(a, b_{t}) \\
&=
f(a^{*}, b^{*})
+
\frac{1}{T} 
\min_{a \in \mathcal{A}}
\sum_{t = 1}^{T} 
[f(a, b_{t})
- 
f(a_{t}, b_{t})]
+
\frac{1}{T}
\sum_{t = 1}^{T}
[f(a_{t}, b_{t})
- 
f(a^{*}, b^{*})] \\
&\geq 
f(a^{*}, b^{*})
-
\frac{3L_{a} D_{a}}{2\sqrt{T}}
-
\frac{3L_{b} D_{b}}{2\sqrt{T}}.
\end{aligned}
\end{align*}
The second inequality follows from concavity, and the final inequality is a result of Lemma~\ref{lemma:OGD} applied to the sequence \(a_{t}\) and equation~\eqref{eqn:GDALower}.
This completes the proof.
\end{proof}

%---------------------------------------------%
%---------------------------------------------%
\section{Additional Lemmas}
\label{sec:AdditionalLemmas}

\begin{lemma}
Define \(D(\lambda) = \alpha^{-1} \expect(R_{Y}(f) - \lambda)_{+} + \lambda,\)
and let \(F_{f}\) denote the cumulative distribution function of \(R_{Y}(f)\).
Then, we have
\[
D'(\lambda) = 1 + \alpha^{-1} (F_{f}(\lambda) - 1).
\]
\label{lemma:CVaRDerivative}
\end{lemma}

\begin{proof}
We compute the derivative directly.
We obtain
\begin{align*}
& \begin{aligned}
D'(\lambda)
&=
1 + \alpha^{-1} 
\lim_{\epsilon \to 0} \frac{1}{\epsilon}\left\{
\expect\left[(R_{Y}(f) - \lambda - \epsilon)_{+} - (R_{Y}(f) - \lambda)_{+}
\right]
\right\} \\
&=
1 + \alpha^{-1}
\lim_{\epsilon \to 0}
\frac{1}{\epsilon}\left\{
\expect\left[-\epsilon \ind\left\{R_{Y}(f) - \lambda > 0\right\}
\right]
\right\} \\
&=
1 - \alpha^{-1} \expect \ind\left\{R_{Y}(f) > \lambda\right\} \\
&=
1 + \alpha^{-1}(F_{f}(\lambda) - 1).
\end{aligned}
\end{align*}
This completes the proof.
\end{proof}

\begin{lemma}
Define \(H(\lambda) = \expect\left[\alpha^{-1}(R_{Y} - \lambda)_{+}\right] + \lambda\).
Then, the derivative of \(H(\lambda)\) is
\[
H'(\lambda)
=
1 - \expect\left[\alpha_{Y}^{-1} \ind\left\{R_{Y}(f) > \lambda\right\}\right].
\]
\label{lemma:HCVaRDerivative}
\end{lemma}

\begin{proof}
We again compute directly, obtaining
\begin{align*}
& \begin{aligned}
H'(\lambda)
&= 
1 
+
\lim_{\epsilon \to 0}
\frac{1}{\epsilon} \expect\left[ 
\alpha_{Y}^{-1}(R_{Y}(f) - \lambda - \epsilon)_{+} - \alpha_{Y}^{-1}(R_{Y}(f) - \lambda)_{+} 
\right] \\
&=
1 + 
\lim_{\epsilon \to 0} \frac{1}{\epsilon}
\expect\left[
\alpha_{Y}^{-1}(-\epsilon) \ind\left\{R_{Y}(f) > \lambda\right\}
\right] \\
&=
1 
-
\expect\left[
\alpha_{Y}^{-1} \ind\left\{R_{Y}(f) > \lambda\right\}
\right],
\end{aligned}
\end{align*}
as desired.
\end{proof}

%---------------------------------------------%
%---------------------------------------------%
\begin{lemma}
We have the inequality
\[
\inf_{q \in Q} \left\{A(q) + B(q)\right\}
\leq 
\inf_{q \in Q} A(q) + \sup_{q \in Q} B(q).
\]
\label{lemma:SimpleInfimum}
\end{lemma}

\begin{proof}
We have the inequality \(A(q) + B(q) \leq A(q) + \sup_{q' \in Q} B(q')\), and taking infimums completes the proof.
\end{proof}

%---------------------------------------------%
%---------------------------------------------%
\section{Standard Lemmas}
\label{sec:StandardLemmas}

\begin{lemma}[Theorem~3.1 of \citealt{mohri2012}]
Let \(G\) be a family of functions mapping from \(\reals\) to \([0, 1]\).
Then for \(\delta > 0\) and all \(g\) in \(G\), with probability at least \(1 - \delta\), we have
\begin{align*}
&\begin{aligned}
\expect g(Z)
&\leq 
\frac{1}{n}
\sum_{i = 1}^{n} g(Z_{i})
+
2 \rademacher_{n}(G)
+
\sqrt{\frac{\log\frac{1}{\delta}}{2n}}.
\end{aligned}
\end{align*}
\label{lemma:StandardRademacher}
\end{lemma}

For our excess \((\functions, q)\)-risk bounds, we also use a slight variant, the proof of which is nearly identical to that of Lemma~\ref{lemma:StandardRademacher}.
\begin{lemma}
Let \(G\) be a family of functions mapping from \(\reals\) to \([0, 1]\).
Then for \(\delta > 0\) and all \(g\) in \(G\), with probability at least \(1 - \delta\), we have
\begin{align*}
&\begin{aligned}
\left|
\expect g(Z)
-
\frac{1}{n}
\sum_{i = 1}^{n} g(Z_{i})
\right|
&\leq 
4 \rademacher_{n}(G)
+
\sqrt{\frac{\log\frac{1}{\delta}}{2n}}.
\end{aligned}
\end{align*}
\label{lemma:ModifiedRademacher}
\end{lemma}

The following learning bound handles the multi-class margin loss more effectively in the number of classes \citep{kuznetsov2015}.
\begin{lemma}
Let \(\functions\) be a set of \(f: \xspace \times \yspace \to \reals\).
Recall that 
\[
\Pi_{1}(\functions)
=
\left\{x \mapsto f_{y}(x): y \in \yspace, f \in \functions\right\}.
\]
Then, under the margin loss, we have the bound
\[
\risk(f)
\leq 
\emprisk(f) + 4k \rademacher_{n}(\Pi_{1}(\functions)) + \sqrt{\frac{\log \frac{1}{\delta}}{2n}}
\]
for all \(f\) in \(\functions\) with probability at least \(1 - \delta\).
\label{lemma:KuznetsovRademacher}
\end{lemma}

\section{Additional Experiment Details}
\label{sec:AdditionalExp}

For all methods and datasets, we optimized a logistic regression model with gradient descent over the entire data. 

For all datasets, we chose a learning rate of 0.01 that was linearly annealed to 0.0001 over 2000 epochs. 

\subsection{Optimizing LCVaR/LHCVaR formulation}

Note that in the formulation for LHCVaR described in \cref{eqn:EmpiricalLHCVaR}, despite its convexity, the optimization is over a non-smooth loss. Thus, \(\lambda\) can be explicitly calculated given the classes of each risk. Let \(R_{(i)}\) be the \(i\)th largest class risk.

\begin{align*}
	\lambda = \min\ \left(\left\{R_{(i)}: i \in [k], \sum\limits_{j = 1}^{i}\hat{p}_i\alpha_i^{-1} \leq 1\right\} \cup \{0\}\right)
\end{align*}

An algorithm for computing this can be akin to water filling in order from largest to smallest class risk. When optimizing by some form of gradient descent the parameters of the classifier, this analytic form of the LHCVaR formulation can be quickly computed and avoid gradient computations on \(\lambda\) itself. Empirically, we used this formulation to speed up our experiments and leads to faster convergence than performing gradient descent on \(\lambda\) in addition to the model parameters. This algorithm is also applicable when optimizing LCVaR as well.
% flatex input end: [tex/appendix.tex]

%---------------------------------------------%
\end{document}
% flatex input end: [arxiv_draft.tex]